\documentclass[11pt]{article}

\usepackage{amsthm,amsmath,amssymb}
\usepackage{natbib}
\usepackage{multirow}
\usepackage[pdftex]{graphicx}
\usepackage{subfigure}
\usepackage{makecell}
\usepackage{booktabs}
\usepackage{array}
\usepackage{fullpage}
\usepackage{url} 
\usepackage{algorithm}
\usepackage{algorithmic}
\usepackage{bm}
\usepackage{smile-1}
\usepackage{mathtools}
\usepackage{wrapfig}
\usepackage{lipsum}
\usepackage{mathrsfs}
\usepackage{amssymb}
\usepackage{dsfont}
\usepackage{appendix}

\usepackage{color}
\definecolor{blue}{rgb}{0,0, 0.9}
\definecolor{red}{rgb}{0.9,0,0}
\definecolor{green}{rgb}{0, 0.9,0}

\usepackage{natbib}
\usepackage{multirow}
\usepackage[usenames,dvipsnames,svgnames,table]{xcolor}
\usepackage[colorlinks=true,
            linkcolor=blue,
            urlcolor=blue,
            citecolor=blue]{hyperref}

\theoremstyle{definition}
\newtheorem{thm}{Theorem}[section]
\newtheorem{defi}[thm]{Definition}
\newtheorem{rmk}[thm]{Remark}
\newtheorem{prp}[thm]{Proposition}

\newtheorem{lem}[thm]{Lemma}

\def\[{\lf[} \def\]{\ri]}   \def\lf{\left} \def\ri{\right}

\def\inprod#1#2{\langle #1,\, #2\rangle}

\newcommand{\brr}{\mathbb{R}}
\newcommand{\e}{\mathbb{E}}
\newcommand{\p}{\mathbb{P}}
\newcommand{\sn}{\sum_{i=1}^n}
\renewcommand{\S}{\mathbb{S}}
\renewcommand{\cS}{\mathbb{S}}

\renewcommand{\baselinestretch}{1.1}
\numberwithin{equation}{section}

\begin{document}
\title{\huge {Max-Norm Optimization for Robust Matrix Recovery}}
\author{Ethan X. Fang\thanks{Department of Statistics, Department of Industrial and Manufacturing Engineering, Pennsylvania State University, University Park, PA 16802, USA. E-mail: {\tt
xxf13@psu.edu}}
~~~Han Liu\thanks{Department of Operations Research and Financial Engineering, Princeton University, Princeton, NJ 08544, USA. E-mail: {\tt
hanliu@princeton.edu}}
~~~Kim-Chuan Toh\thanks{National University of Singapore, 10 Lower Kent
Ridge Road, Singapore 119076. Research supported in part by Ministry of Education
Academic Research Fund R-146-000-194-112. E-mail:
{\tt
mattohkc@nus.edu.sg}}
 ~~~Wen-Xin Zhou\thanks{Department of Operations Research and Financial Engineering, Princeton University, Princeton, NJ 08544, USA. E-mail: {\tt
wenxinz@princeton.edu}}
 }
 \date{}
\maketitle
\begin{abstract}

This paper studies the matrix completion problem under arbitrary sampling schemes. We propose a new estimator incorporating both max-norm and nuclear-norm regularization, based on which we can conduct efficient low-rank matrix recovery using a random subset of entries observed with additive noise under general non-uniform and unknown sampling distributions. 
This method significantly relaxes the uniform sampling assumption imposed for the widely used nuclear-norm penalized approach, and makes low-rank matrix recovery feasible in more practical settings. Theoretically, we prove  that the proposed estimator {achieves} fast rates of convergence under different settings. Computationally, we propose an alternating direction method of multipliers algorithm to efficiently compute the estimator, which bridges a gap between theory and practice of machine learning methods with max-norm regularization. Further, we provide thorough numerical studies to %backup
{evaluate} the proposed method using both simulated and real~datasets.

\end{abstract}
%However, solving the max-norm minimization problem is challenging. We propose an alternating direction of method of multipliers (ADMM)-based method to efficiently solve the optimization problem. d
%Our algorithm bridges the gap between theoretical guarantees and practical implementations of the max-norm approach. \end{abstract}

%Note: The file ``MaxNormLS.m" solves the max-norm penalized optimization problem \eqref{eqn:maxnorm}, and the file ``solveMax.m" solves the optimization problem of the constraint form as discussed in \cite{cai2013matrix}.

\section{Introduction} \label{sec:int}
We consider the matrix completion problem, which aims to reconstruct an unknown matrix based on a small number of entries contaminated by additive noise. This problem has drawn significant attention over the past decade due to its wide  applications, including collaborative filtering (the well-known Netflix problem) \citep{Netflix,bennett2007netflix}, multi-task learning \citep{abernethy2009new,amit2007uncovering,argyriou2008convex}, sensor-network localization %\citep{toh2006sensor,candes2010matrix} 
\citep{toh2006sensor}
and system identification \citep{liu2009interior}. 
Specifically, our goal is to recover an unknown matrix $M^0 \in \RR^{d_1\times d_2}$ based on a subset of its entries observed with noise, say $\{Y_{i_t,j_t}\}_{t=1}^n$. In general, the problem of recovering a partially observed matrix is ill-posed, {as the unobserved entries can take any values without further assumption}. However, in many applications mentioned above, it is natural to impose the condition that the target matrix is of either exact or approximately low-rank, which {avoids the ill-posedness} and makes the recovery possible.

To obtain a low-rank estimate of the matrix, a straightforward approach is to consider the rank minimization problem
\begin{equation}\label{eqn:rank}
\min_{M\in\RR^{d_1\times d_2}} \text{rank}(M), \text{ subject to }\|Y_\Omega - M_\Omega\|_F\le \delta,
\end{equation}
where $\Omega = \{(i_t, j_t): t = 1,\ldots,n\}$ is the index set of observed entries, and $\delta>0$ is a tuning parameter. 
This method directly searches for a matrix of the lowest rank with reconstruction error controlled by $\delta$. However, the optimization problem \eqref{eqn:rank} is computationally intractable due to its nonconvexity. A commonly used alternative is the following convex relaxation of \eqref{eqn:rank}:
\begin{equation}\label{eqn:nuclear}
\min_{M\in\RR^{d_1\times d_2}} \|M\|_*, \text{ subject to }\|Y_\Omega - M_\Omega\|_F\le \delta,
\end{equation}
where $\|\cdot \|_*$ denotes the nuclear-norm (also known as the trace-norm, Ky Fan-norm or Schatten 1-norm), and it  is defined as the sum of singular values of a matrix. Low-rank matrix recovery based on nuclear-norm regularization has been extensively studied in both noiseless and noisy cases \citep{candes2009exact,candes2010power,recht2010guaranteed,koltchinskii2011nuclear,rohde2011estimation,recht2011simpler,keshavan2010matrix,negahban2012restricted}. Furthermore, various computational algorithms have been proposed to solve this  problem. For example, \cite{cai2010singular} propose a singular value thresholding algorithm which is equivalent to the gradient method for solving the dual of a regularized version of \eqref{eqn:nuclear}; \cite{toh2010accelerated} propose an accelerated proximal gradient
method to solve {a least squares version} of \eqref{eqn:nuclear};  \cite{liu2009interior} exploit an interior-point method; \cite{chen2012matrix} adopt an alternating direction method of multipliers approach to solve \eqref{eqn:nuclear}. 

Though significant progress has been made, it {remains} unclear whether the nuclear-norm is the best convex relaxation for the rank minimization problem \eqref{eqn:rank}. Recently, some disadvantages of the nuclear-norm regularization have been noted. For instance, the theoretical guarantee of the nuclear-norm regularization relies on an
assumption that the  indices of the observed entries  are uniformly sampled. That is, each entry is equally likely to be observed as illustrated in Figure \ref{fig:net}(a). This assumption is restrictive in applications. Taking the well-known Netflix problem as an example, our goal is to reconstruct a movie-user rating matrix, in which each row represents a user and each column represents a movie. The $(k, \ell )$-th entry of the rating matrix represents the $k$-th user's rating for the $\ell$-th movie. In practice, we only observe a small proportion of the entries.
In this example, the uniform sampling assumption is arguably violated due to the following reasons: (1) Some users are more active than others, and they rate more movies than others. (2) Some movies are more popular than others and are rated by more users.  As a consequence,  the entries from certain columns or rows are more likely to be observed. See Figure \ref{fig:net}(b) for a simple illustration. To sum up, the sampling distribution can be highly non-uniform in real world applications.

\begin{figure}[ht]
 \centering
\includegraphics[width=0.8\textwidth]{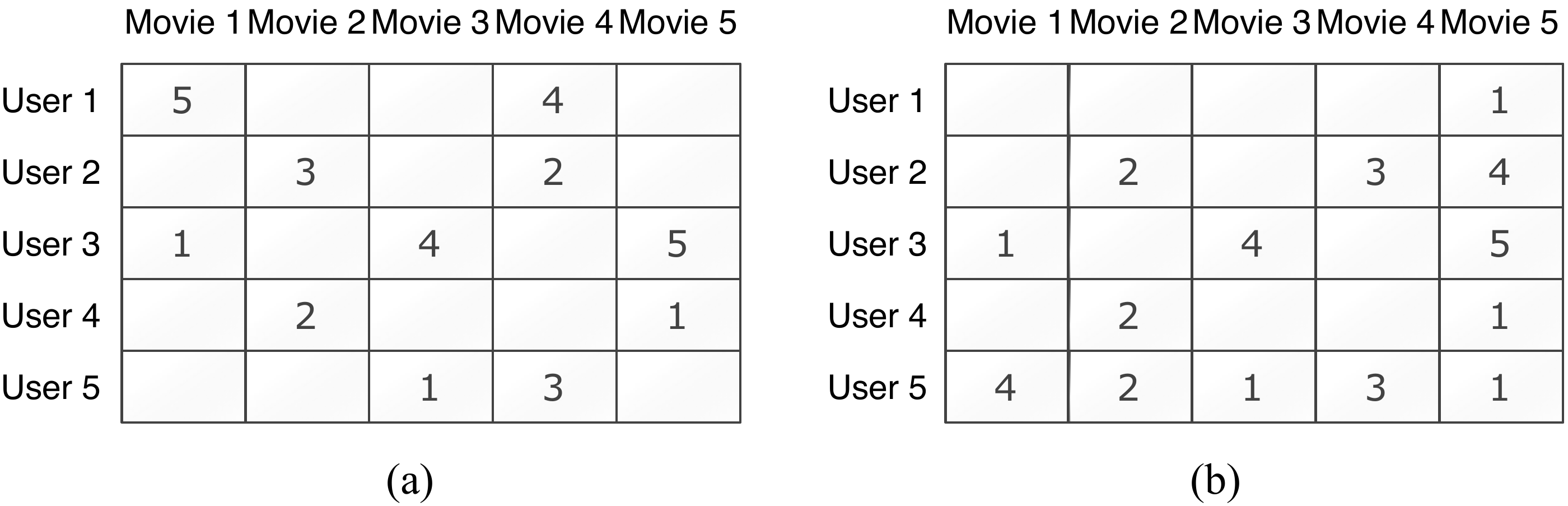}
\caption{(a) The theoretical guarantee of the nuclear-norm estimator assumes each entry is equally likely to be observed. (b) In practice, some entries 
related to some popular movies or some active users, such as Movie 5 or User 5, are more likely to be sampled than others. Thus, the uniform sampling assumption is violated.}
\label{fig:net}
\end{figure}

To relax or even avoid the unrealistic uniform sampling assumption, several recent papers propose to use the matrix max-norm as a convex surrogate for the rank. \cite{srebro2010collaborative} observe from empirical comparisons that the max-norm regularized approach outperforms the nuclear-norm based one for matrix completion and collaborative filtering under non-uniform sampling schemes. \cite{lee2010practical} and \cite{jalali2012clustering} demonstrate the advantage of using max-norm regularizer over nuclear-norm in some other applications. More recently, \cite{cai2013matrix} prove that the max-norm regularized estimator is minimax rate-optimal (over a class of approximately low-rank matrices) {under non-uniform sampling schemes}.

Though the max-norm approach possesses attractive theoretical properties, efficiently solving large-scale max-norm optimization problem remains challenging and prevents the wide adoption of max-norm regularizer.  As we shall see later, despite the fact that the max-norm is  a convex regularizer and can be formulated as a semidefinite programming problem, classical methods such as interior-point methods are only scalable to moderate dimensions, while the problem of {practical interest is of large dimensions. In recent work, 
\cite{lee2010practical} and \cite{shen2014online} propose first-order algorithms for a nonconvex relaxation of the problem.
 However, these methods are sensitive to  the choice of  initial points and stepsizes, and are only %guaranteed to produce local solutions, 
{capable of producing stationary solutions,}
whose  statistical properties remain open due to the  nonconvexity. 
%In this paper, we try to achieve optimal low-rank matrix recovery under both uniform and non-uniform sampling schemes with a {computationally} efficient approach.
%We consider general unspecified, non-uniform sampling scheme. The max-norm approach obtains 
%{the setting we consider is more complex than other existing work},
%we are essentially dealing with a much more complex model than other existing work, 
%and the max-norm constraint is not only introduced as a convex relaxation for low-rankness but also takes into account the effect of non-uniform sampling. 
Meanwhile, although the max-norm estimator is adaptive to general sampling schemes,  it was shown in \cite{cai2013matrix} that if the target matrix is of exact low-rank, and the sampling scheme is uniform, the max-norm estimator only achieves a sub-optimal rate compared to the nuclear-norm estimator. Specifically, letting $\hat{M}_{\max}$ and $\hat{M}_*$ be  the estimators using max-norm and nuclear-norm regularizers,  we have $$(d_1 d_2)^{-1} \|\hat{M}_{\max} - M^0\|_F^2 = \cO_{\PP}( n^{-1/2} \sqrt{rd} ) ~\text{and} ~~ (d_1 d_2)^{-1} \|\hat{M}_* - M^0 \|_F^2 = \cO_{\PP}  (n^{-1}rd \log d),$$ where $r$ is the rank of $M^0$  and $d=d_1 + d_2$. To compare, under the uniform sampling scheme, the nuclear-norm regularized method  achieves the optimal rate of convergence (up to a logarithmic factor) and is computationally more scalable. 

To achieve the advantages of both regularizers,  we propose a new estimator using a hybrid regularizer. %, which can be computed efficiently. 
Meanwhile, we propose an efficient alternating direction method of multipliers (ADMM) algorithm to solve the optimization problem.  Our method includes the max-norm regularizer as a special case, and the proposed algorithm is scalable to  {modestly} large dimensions.
The contribution of this paper is  two-fold: First, we propose an estimator for matrix completion under genearal sampling scheme, which achieves optimal rate of convergence in the exact low-rank case and is adaptive to different sampling schemes. Second, we provide an efficient algorithm to solve the corresponding max-norm plus nuclear-norm penalized optimization problem.  We illustrate  {the} efficiencies of the proposed {methods and algorithms}  by  numerical experiments on both simulated and real datasets.

%In particular, 

\vspace{0.1in}
\noindent {\bf Notation.} Throughout this paper, we adopt the following notations. For any positive integer $d$, $[d]$ denotes the set of integers $\{1,2,\ldots,d\}$. For a vector $v = (v_1,\ldots,v_d)^T \in \RR^d$ and a positive number $p\in(0,\infty)$, we denote $\|u\|_p$ 
as  the $\ell_p$-norm, i.e., $\|u\|_p  = \big(\sum_{i=1}^{d}|u_i|^p\big)^{1/p}$. Also, we let $\|u \|_\infty = \max_{i=1,\ldots,d}|u_i|$. For a matrix $M=(M_{k \ell}) \in\RR^{d_1\times d_2}$, let $\|M\|_F = \big({\sum_{k=1}^{d_1} \sum_{\ell=1}^{d_2}M_{k \ell}^2}\big)^{1/2}$ be the Frobenius-norm, and we denote the matrix elementwise $\ell_\infty$-norm by $\|M\|_\infty = \max_{k,l}|M_{k \ell}|$. Given the $\ell_p$ and $\ell_q$ norms on $\RR^{d_1}$ and $\RR^{d_2}$, we define the corresponding $\|\cdot\|_{p,q}$ operator-norm, where $\|M\|_{p,q} = \sup_{\|x\|_p = 1} \|Mx\|_q$. For examples, $\|M\| = \|M\|_{2,2}$ is the spectral-norm, and $\|M\|_{2,\infty} = \max_{k=1,\ldots,d_1} \big({\sum_{\ell=1}^{d_2}M_{k \ell}^2}\big)^{1/2}$ is the maximum row norm of $M$. {We denote by $a\asymp b$ if $c_1b\le a \le c_2b$ for two constants $c_1$ and $c_2$.}%Finally, we denote by $[d] = \{1,...,d\}$.

\vspace{0.1in}

\noindent {\bf Paper Organization.} The rest of this paper is organized as follows. In Section \ref{sec:pre}, we review the max-norm approach and formulate the problem. In Section \ref{sec:alg}, we propose the algorithm. In Section \ref{sec:thm}, we provide  theoretical analysis of the estimator. 
We provide extensive numerical studies in Section \ref{sec:sim}, and we conclude the paper in Section \ref{sec:con}.

\section{Preliminaries and  Problem Formulation} \label{sec:pre}

In this section, we first introduce the concept of the matrix max-norm \citep{linial2007complexity}. Next, we  propose a new estimator which involves both max-norm and nuclear-norm regularizers.

\begin{defi} \label{def:maxnorm}
The max-norm of a matrix $M\in\RR^{d_1\times d_2}$ is defined as 
$$
\|M\|_{\max} = \min_{M=UV^T} \|U\|_{2,\infty} \|V\|_{2,\infty},
$$
where the minimum is over all factorizations $M=UV^T$ for $U\in\RR^{d_1\times k}$, $V\in\RR^{d_2\times k}$ for $k = 1, \ldots , \min(d_1, d_2)$, and $\|U\|_{2,\infty}$, $\|V\|_{2,\infty}$ denote the operator-norms {of} $U:\ell_2^k\rightarrow\ell_\infty^{d_1}$ and $V:\ell_2^k\rightarrow\ell_\infty^{d_2}$.
\end{defi} 

We briefly compare the max-norm and nuclear-norm regularizers. We refer to \cite{srebro2005rank} and \cite{cai2013matrix} for more detailed discussions. Recall that the nuclear-norm of the matrix $M$ is defined as
$$
\|M\|_* = \min\Big\{\sum_{j}|\sigma_j|: M = \sum_{j}\sigma_ju_jv_j^T,u_j\in\RR^{d_1}, v_j\in\RR^{d_2}, \|u_j\|_2 = \|v_j\|_2 = 1\Big\}.
$$
From the definition, the nuclear-norm encourages low-rank approximation with factors in the $\ell_2$-space. On the other hand, it is known \citep{jameson1987summing} that the max-norm has a similar interpretation by replacing the {constraints} in the $\ell_2$-space by those in the $\ell_\infty$-space:
$$
\|M\|_{\max} \asymp \min\Big\{\sum_{j}|\sigma_j|: M = \sum_{j}\sigma_ju_jv_j^T,u_j\in\RR^{d_1}, v_j\in\RR^{d_2}, \|u_j\|_\infty = \|v_j\|_\infty = 1\Big\},
$$
where the factor of equivalence is the Grothendieck's constant $K\in (1.67,1.79)$. {More specifically, a consequence of Grothendieck's inequality is that $K_{{\rm G}}^{-1} \| M \|_{1\to \infty} \leq \| M \|_{\max} \leq \| M \|_{1\to \infty}$ \citep{srebro2005rank}, where $\| M \|_{1\to \infty} := \max_{ u \in \RR^{d_2}: \| u \|_1 \leq 1} \| M u \|_\infty$ for any $M\in \RR^{d_1 \times d_2}$.} This gives some intuition on why the max-norm regularizer could  outperform the nuclear-norm regularizer when the matrix entries  are uniformly bounded. This scenario indeed {stands} in many  applications. For example, in the Netflix problem or the  low-rank correlation matrix estimation problem, the entries of the unknown matrix are either ratings or correlation coefficients, and  are uniformly bounded.

As mentioned in Section \ref{sec:int}, the advantages of {using the} max-norm over {the} nuclear-norm  are well illustrated in the literature from both theoretical and practical perspectives. Specifically, we consider the matrix completion problem in a general sampling scheme. Let $M^0\in\RR^{d_1\times d_2}$ denote the unknown matrix to be recovered. {Assume that we are given a random  index set $\Omega$ of size $n$:}
$$
\Omega = \big\{  (i_t,j_t):t=1,\ldots,n  \big\} \subset\big( [d_1]\times[d_2]   \big)^n,
$$
{where $[d_i]=\{1,2,\ldots,d_i\}$ for $i=1,2$.}
We further assume that the samples of the indices are drawn independently from a general sampling distribution ${\Pi} =  \{ \pi_{k \ell}  \}_{k\in[d_1],  \ell  \in[d_2]}$ on $[d_1]\times[d_2]$. Note that we consider the sampling scheme with replacement, i.e., we assume $\PP \{ (i_t,j_t) = ( k, \ell ) \} =\pi_{k \ell}$ for all $t\in [n]$ and all $(k, \ell )\in [d_1]\times[d_2]$. For example, the sampling scheme is uniform if $\pi_{k \ell} =  ( d_1d_2 )^{-1}$ for all $(k, \ell )\in [d_1]\times [d_2]$. %Thus, the sampling scheme considered here is much more general. 
Given the sampled index set $\Omega$, we further observe noisy entries $\{Y_{i_t,j_t}\}_{t\in [n]}$: % given by
$$
Y_{i_t,j_t} = M^0_{i_t,j_t} + \sigma\xi_t, \text{ for }t = 1, \ldots , n,
$$
where $\sigma>0$ denotes the noise level, and $\xi_t$'s are independent and identically distributed random variables  with $\EE(\xi_t) = 0$ and $\EE(\xi_t^2) = 1$.

Using the max-norm regularization, \cite{cai2013matrix} propose to construct an estimator 
\begin{equation}\label{eqn:mhat}
\hat{M}_{\max} = \argmin_{M\in\RR^{d_1\times d_2}} \frac{1}{n}\sum_{t=1}^n \big( Y_{i_t,j_t} - M_{i_t,j_t}  \big)^2, \text{ subject to }M\in \cK(\alpha,R),
\end{equation}
where $\cK(\alpha,R) = \big\{M\in\RR^{d_1\times d_2}: \|M\|_\infty \le \alpha, \|M\|_{\max}\le R \big\}$ with $\alpha$ {being} a prespecified upper bound for the elementwise $\ell_\infty$-norm of $M^0$ and $R>0$ a tuning parameter. Note that, in {many real world} applications, we have a tight upper bound on the magnitudes of all the entries of $M^0$ in advance. This condition enforces that $M^0$ should not be too ``spiky'', and a %too 
{loose upper} 
bound may jeopardize {the} 
 estimation accuracy \citep{negahban2012restricted}. Also, the recent work by \cite{lee2010practical} argues {that} 
the max-norm regularizer {produces} better empirical results on low-rank matrix recovery for uniformly bounded data. %This is another reason that max-norm regularization is preferred \citep{cai2013matrix}.

\cite{cai2013matrix} provide  theoretical guarantees for the max-norm regularizer \eqref{eqn:mhat}. Specifically, {under the approximately low-rank assumption that $\| M^0 \|_{\max } \leq R$},  we have, %with probability at least $1-2e^{-d}$,
$$
\frac{1}{d_1d_2}\|\hat{M}_{\max} - M^0\|_F^2  = \cO_{\PP}\Big(\sqrt{ \frac{R^2 d}{n}} \,\Big),
$$
where $d = d_1+d_2$. This rate matches the minimax lower bound over all approximately low-rank matrices even under  non-uniform sampling schemes. %, indicating that the max-norm constrained method is minimax rate-optimal. 
See \cite{cai2013matrix} for more  details.

The optimization problem  \eqref{eqn:mhat}  is computationally challenging.
\cite{cai2013matrix} employ a first-order method proposed in \cite{lee2010practical}. {In particular, \cite{lee2010practical} and \cite{shen2014online} consider first-order methods based on rewriting problem \eqref{eqn:mhat} into the following form:
$$
\min_{U, V} \frac{1}{n}\sum_{t=1}^n(U_{i_t}^TV_{j_t} - Y_{i_t,j_t})^2, \text{ subject to } \|U\|_{2,\infty}^2 \le R,\ \|V\|_{2,\infty}^2 \le R,\ \max_{(k,\ell)\in[d_1]\times[d_2]}|U_k^TV_\ell|\le \alpha,
$$
where $U_i$ and $V_j$ denote the $i$-th row of $U$ and the $j$-th row of $V$, respectively. Then, \cite{lee2010practical} and \cite{shen2014online} consider different efficient first-order methods to solve this problem. 
 However, the problem is nonconvex, and the convergence behaviors  of those methods on such a nonconvex problem are generally sensitive to  
the  choice of the initial point and stepsize selection. More seriously, the algorithms mentioned can only guarantee  local stationary solutions, which may not necessarily possess the nice theoretical properties for the solution to problem 
\eqref{eqn:mhat}.}
More recently, \cite{orabona2012prisma} solve the optimization problem  \eqref{eqn:mhat} without the uniform-boundedness constraint. However, it is unclear %that 
how to extend their algorithms to solve the  problem \eqref{eqn:mhat} with the  $\ell_\infty$-norm constraint. %In the next two sections, we show {that} our work solves the problem which is efficient while guarantees the convergence to a global minimizer. 

In the next section, we aim to solve the {max-norm penalized} optimization problem
\begin{equation}	\label{eqn:maxnormp}
\min_{M\in\RR^{d_1\times d_2}} \frac{1}{n}\sum_{t=1}^n \big( Y_{i_t,j_t} - M_{i_t,j_t}  \big)^2 + \lambda\|M\|_{\max}, \text{ subject to }\|M\|_\infty\le \alpha,
\end{equation}
where $\lambda>0$ is a tuning parameter. 
By convexity and strong duality, the problem \eqref{eqn:maxnormp} is equivalent to \eqref{eqn:mhat} for a properly chosen $\lambda$.  Specifically, for any $R$ specified in \eqref{eqn:mhat}, there exists a $\lambda$ such that the solutions to the two problems coincide.

%\cite{cai2013matrix} adopt a first-order method  proposed in \cite{lee2010practical} to solve the optimization problem \eqref{eqn:mhat}. However, as we discussed in the introduction, this method is derived by some nonconvex relaxation. Such method is sensitive with respect to the choice of starting point and stepsize. Some other works \citep{orabona2012prisma} solve the optimization problem \eqref{eqn:mhat} without the uniform-boundedness constraint, and it is unclear how to extend their algorithms to solve the two problems above. In the next two sections, we show {that} our work solves the problem which is efficient while guarantees the convergence to a global minimizer. % and we demonstrate the algorithm works well empirically.

As discussed in Section \ref{sec:int}, a major drawback of the max-norm penalized estimator \eqref{eqn:mhat} is that 
 if the {underlying} true matrix $M^0$ is of exact low-rank, and when the sampling scheme is indeed uniform, the max-norm regularizer  does not perform as well as  {the} nuclear-norm regularizer. Since the underlying structure of $M^0$ and the sampling scheme are unknown, it is difficult to choose the better approach in practice. To overcome this issue, we propose the following hybrid estimator which is expected to be more flexible and adaptive: % in practice: %, which obtains fast rate of convergence under the ideal situation, and is robust to the sampling scheme:
\begin{equation} \label{eqn:hybrid}
\widehat{M}: = \argmin_{M \in \RR^{d_1\times d_2} }  \frac{1}{n}\sum_{t=1}^n \big( Y_{i_t,j_t} - M_{i_t,j_t}  \big)^2 + \lambda\|M\|_{\max} + \mu \|M\|_*, \text{ subject to }\|M\|_\infty\le \alpha
\end{equation}
{where $\mu$ is  a nonnegative tuning parameter. The addition of the nuclear-norm penalization
is motivated by the fact that the nuclear-norm also serves as a convex surrogate for the rank of the estimator. Thus, the addition of the nuclear-norm encourages the estimator to be low rank or approximately low rank as compared to the max-norm estimator in \eqref{eqn:maxnormp}. However, note that our primary goal here is not to find a low-rank estimator but one which approximates the underlying matrix $M^0$ at near optimal recovery and is robust against the unknown sampling scheme. It is worth mentioning that the use of the sum of two norms in matrix recovery has been considered in other contexts. For example, in robust principal component analysis \cite{Candes2009RPCA}, the sum of the nuclear and $\ell_1$ norms is used in  the recovery of the low-rank and sparse components of a given superposition. In \cite{DoanVavasis2013}, a similar combination of the two norms (denoted as $\norm{\cdot}_{1,*} := \norm{X}_1+ \theta \norm{X}_*$ for a given matrix $X$ and a parameter $\theta$) is used to find hidden sparse rank-one matrices in a given matrix. The geometry of the unit $\norm{\cdot}_{1,*}$-norm ball is further analyzed in  \cite{drusvyatskiy2015extreme}. It is interesting to note that \eqref{eqn:hybrid} is the first time that the sum of the max-norm and nuclear norm is considered in matrix recovery.

In Section \ref{sec:alg}, we propose an efficient algorithm to solve \eqref{eqn:hybrid}, which includes \eqref{eqn:maxnormp} as a special case by taking $\mu = 0$. Section \ref{sec:thm} provides theoretical justification for the hybrid estimator $\widehat{M}$ in \eqref{eqn:hybrid}. In particular, it achieves fast rate of convergence under the ``ideal" situation, and is robust against non-uniform sampling schemes. To sum up, this estimator possesses {the}
advantages of both {the} max-norm and nuclear-norm regularizers. Section \ref{sec:sim} provides empirical results of the algorithm.
}

\section{Algorithm} \label{sec:alg}
%As discussed in \cite{srebro2004maximum}, max-norm regularized optimization problems can be formulated as semidefinite {programming}  problems (SDP). However, generic SDP solvers {based on interior-point methods}  cannot always solve those problems of practical interests. For example, the Netflix problem is of very large dimensions. Directly applying any state-of-the-art SDP solver does not scale for problems of such large sizes. 
In this section, we propose a new algorithm to solve the  problem \eqref{eqn:hybrid}. 
The key step is to reformulate the problem to expose the structure.
%{for which we can design an ADMM algorithm to solve the problem efficiently.}
%Next, we apply the ADMM to solve the problem efficiently.

\subsection{Algorithmic Framework}
We first review that the max-norm regularized problem \eqref{eqn:maxnormp} can be equivalently formulated as a semidefinite programming (SDP) problem.  By Definition~\ref{def:maxnorm},  it is unclear  how to efficiently compute the max-norm of a given matrix. 
By \cite{srebro2004maximum}, the max-norm of a matrix $A$ can be computed via solving the following SDP problem:
$$
\|A\|_{\max} = \min R, \text{ subject to }
\begin{pmatrix}
W_1 & A\\
A^T & W_2
\end{pmatrix}
\succeq 0,\ 
\|\diag(W_1)\|_\infty\le R,\ \|\diag(W_2)\|_\infty \le R.
$$
Thus, the max-norm penalized problem \eqref{eqn:maxnormp} can be formulated as an SDP problem that 
%{as follows:}
\begin{equation}\label{eqn:sdp0}
\begin{aligned}
\min_{Z\in \RR^{d \times d}} &~ \frac{1}{2}\sum_{t=1}^n(Y_{i_t,j_t} - Z^{12}_{i_t,j_t})^2 + \lambda\, { \|\diag(Z)\|_\infty},
\\
\text{subject to }&~\|Z^{12}\|_\infty\le \alpha,\ Z\succeq 0,
\end{aligned}
\end{equation}
where $d = d_1 + d_2$, and
$$
Z = \begin{pmatrix}
Z^{11}&Z^{12}\\
(Z^{12})^T&Z^{22}
\end{pmatrix}, \ Z^{11}\in\RR^{d_1\times d_1}, \ Z^{12}\in\RR^{d_1\times d_2}\text{ and }Z^{22}\in\RR^{d_2\times d_2}.
$$
{One may observe that the problem \eqref{eqn:sdp0}  does not explicitly 
encourage the optimal solutions to be low-rank matrices, although such a property is 
desirable in many practical applications such as collaborative filtering. Thus, we propose
to add the regularization term involving $\inprod{I}{Z}$, which is the convex 
surrogate for the rank of the  positive semidefinite matrix $Z$,  to the objective function in 
\eqref{eqn:sdp0} to obtain the following hybrid optimization problem:}
\begin{equation}\label{eqn:sdp}
\begin{aligned}
\min_{Z\in \RR^{d \times d}} &~\frac{1}{2}\sum_{t=1}^n(Y_{i_t,j_t} - Z^{12}_{i_t,j_t})^2 + \lambda\, { \|\diag(Z)\|_\infty} + {\mu \inprod{I}{Z}},
\\
\text{subject to }&~ \|Z^{12}\|_\infty\le \alpha,\ Z\succeq 0,
\end{aligned}
\end{equation}
where $\mu \geq 0$ is a tuning parameter. {Note that the estimator in \cite{cai2013matrix} is constructed by solving a special case of this problem by setting $\mu = 0$.}

\begin{rmk} 
The problem \eqref{eqn:sdp} is equivalent to the problem \eqref{eqn:hybrid}. To see this, by Lemma 1 of \cite{fazel2001rank}, there exists an SDP formulation of the trace-norm
{such} that $\|M\|_*\le t$ if and only if there exist matrices $Z^{11}\in\RR^{d_1\times d_1}$, $Z^{12} \in \RR^{d_1\times d_2}$ and $Z^{22}\in\RR^{d_2\times d_2}$ 
{satisfying}
$$
\begin{pmatrix}
Z^{11}&Z^{12}\\
(Z^{12})^T&Z^{22}
\end{pmatrix}
\succeq 0,
\quad\text{and }\text{Trace}\big(Z^{11}\big) + \text{Trace}\big(Z^{22}\big) \le 2t.
$$
\end{rmk}

The optimization problem \eqref{eqn:sdp} is computationally challenging. Directly solving the problem by  generic interior-point method based SDP solvers is not {computationally} scalable. This is because the problem of interest is often of high dimensions, and the %matrix
 $\ell_\infty$-norm constraint in \eqref{eqn:sdp} induces a large number of constraints in the SDP.  In addition,  the feasible set is in a very complex form as it involves both the positive semidefinite and $\ell_\infty$-norm constraints. 
Although gradient projection methods are the most straightforward methods to use, the complicated feasible set {also} makes them 
difficult to {be applied}.
This is because applying such a  method requires projecting the intermediate solution to the feasible set, but it is unclear how to efficiently {compute} the projection. 
%Furthermore, the subgradient of the objective function in \eqref{eqn:sdp} does not have a closed-form. This further brings  challenges in applying 
%gradient projection methods.

{

To solve the problem efficiently, we  consider an equivalent form of 
 \eqref{eqn:sdp} below. As we shall see immediately, this formulation is crucial for efficiently solving the problem:
\begin{equation} \label{eqn:XZ}
\min_{X,Z} ~ \cL(Z) + \mu\inprod{I}{X},\text{ subject to } X\succeq 0,\ Z\in\cP,\ X-Z = 0,
\end{equation}
where  the function $\cL(Z)$ 
 and the set $\cP$ are defined as follows:
\begin{eqnarray} \label{eqn:L}
\cL(Z) = \frac{1}{2}\sum_{t=1}^n(Y_{i_t,j_t} - Z^{12}_{i_t,j_t})^2 
+ \lambda\,  \|\diag(Z)\|_\infty, \quad
\cP = \{Z\in\cS^{d} : \|Z^{12}\|_\infty \le \alpha\},
\end{eqnarray}
and $\S^{d}$ denotes the set of symmetric matrices in $\RR^{d\times d}$. 
}

Intuitively, the advantage of formulating the problem \eqref{eqn:sdp} into the form of \eqref{eqn:XZ} is that we ``split" the complicated feasible set of \eqref{eqn:XZ} into two parts. In particular, $X$ and $Z$  in \eqref{eqn:XZ} enforce the positive semidefinite constraint and the $\ell_\infty$-norm constraints, respectively. The motivation of this splitting is that though projection onto the feasible set of \eqref{eqn:sdp}, which contains both the semidefinite and $\ell_\infty$-norm constrains, is difficult, we can efficiently %conduct 
{compute the}
projection onto the positive semidefinite set or the $\ell_\infty$-constraint set {individually}. As a result, adopting {an alternating direction approach}, in each step, we only need to project $X$ onto the positive semidefinite cone, %{while keeping $Z$ fixed}, 
and control the $\ell_\infty$-norm of $Z$. {%while keeping $X$ fixed}. 
Meanwhile, we impose an additional constraint $X-Z=0$ to ensure the feasibility of both $X$ and $Z$ to the problem \eqref{eqn:sdp}.

To solve \eqref{eqn:XZ}, we consider the augmented Lagrangian function of \eqref{eqn:XZ} 
{defined by}
$$
L(X,Z;W) = \cL(Z) +{\mu\inprod{I}{X}}
+ \langle W,X-Z\rangle + \frac{\rho}{2}\|X-Z\|_F^2, \text{ $X\in\S^{d}_+$, $\ Z\in\cP$},
$$
where $W$ is the dual variable, and $\S_+^d = \{A\in\cS^d: A\succeq 0 \}$ is the positive semidefinite cone.

Then, we apply the ADMM algorithm to solve the problem \eqref{eqn:XZ}. The algorithm runs iteratively, at the $t$-th iteration, we update $(X,Z;W)$ by 
\begin{equation} \label{eqn:admm}
\begin{aligned}
X^{t+1} &= \argmin_{X\in\S^d_+} L(X,Z^t;W^t) 
= \Pi_{\S_+^{d}}\big\{Z^t -\rho^{-1}{(W^t+\mu I)}\big\},\\
Z^{t+1}& = \argmin_{Z\in\cP} L(X^{t+1},Z;W^t) = 
\argmin_{Z\in\cP} ~\cL(Z) +\frac{\rho}{2} \|Z - X^{t+1}- \rho^{-1}W^t\|_F^2,\\
W^{t+1}& = W^{t} + {\tau} \rho(X^{t+1}-Z^{t+1}),
\end{aligned}
\end{equation}
where $\tau \in (0,(1+\sqrt{5})/2)$ is a step-length parameter which is 
typically chosen to be $1.618$. Here,
{$\Pi_{\S_+^d}(A)$ denotes the projection of the matrix $A\in\S^{d}$} onto the semidefinite cone $\S_+^d$. The worst-case $\cO(t^{-1})$ rate of convergence of ADMM method  is shown, {for example}, in \cite{fang2014gadm}.

%%%%%%%%%%%%%%%%%%%%%%%%%
\subsection{Solving Subproblems}
For fast implementations of the algorithm \eqref{eqn:admm}, it is important to solve the $X$- and $Z$-subproblems of \eqref{eqn:admm} efficiently. For the $X$-subproblem, we have that the minimizer is obtained by truncating all the negative eigenvalues of the matrix $Z^t-\rho^{-1}W^t$ to 0's by Eckart-Young Theorem \citep{trefethen1997numerical}. Moreover, the following proposition provides a  solution to the $Z$-subproblem in \eqref{eqn:admm}, which can be computed efficiently.

\begin{prp}
Let  $\Omega = \{(i_t,j_t)\}_{t=1}^n$ be the index set of observed entries in $M^0$.
%Consider the optimization problem
%$
%\min_{Z}
%$$
 For a given matrix $C\in\RR^{d\times d}$, we have
$$
\cZ(C) = \argmin_{Z\in\cP} ~ \cL(Z) + \frac{\rho}{2} \|Z-C\|_F^2,
$$
where
\begin{equation} \label{eqn:ZC1}
\begin{aligned}
&\cZ(C) = \begin{pmatrix}
\cZ^{11}(C)&\cZ^{12}(C)\\
\cZ^{12}(C)^T&\cZ^{22}(C)
\end{pmatrix},\\
&\cZ_{k \ell}^{12}(C) =
\begin{cases}
\Pi_{[-\alpha,\alpha]} \Big(\frac{Y_{k \ell} + \rho C_{k \ell}^{12}}{1+\rho}\Big), &\quad\text{if }(k, \ell )\in \Omega,\\
\Pi_{[-\alpha,\alpha]} (C_{k \ell}^{12}),&\quad \text{otherwise,}
\end{cases} \\
&\cZ_{k \ell}^{11}(C) = C_{k \ell}^{11}\quad \text{ if } k\neq \ell, \quad \cZ_{k \ell}^{22}(C) =  C_{k \ell}^{22} \quad\text{ if } k\neq \ell, \\
& \diag\big\{\cZ(C)\big\} = \argmin_{z\in\RR^d} ~ \lambda\|z\|_\infty + \frac{\rho}{2}\|\diag(C) - z\|^2_2,
\end{aligned}
\end{equation}
and $\Pi_{[a,b]}(x) = \min\{b,\max(a,x) \}$ projects $x\in\RR$ to the interval $[a,b]$.
\end{prp}
\begin{proof}
By the definition of $\cL(Z)$ in \eqref{eqn:sdp}, we have
$$
\cZ(C) = \argmin_{Z\in\cP} ~ \frac{1}{2}\sum_{t=1}^n(Z^{12}_{i_t,j_t} - Y_{i_t,j_t})^2+ \lambda\|\diag(Z)\|_\infty^2 + \frac{\rho}{2} \|Z-C\|_F^2.
$$
This optimization problem is equivalent to 
\begin{equation}\label{eqn:equiv}
\begin{aligned}
&\min_{Z^{12}} \Big\{\frac{1}{2}\sum_{t=1}^n(Z^{12}_{i_t,j_t} - Y_{i_t,j_t})^2  + \rho \|Z^{12}-C^{12}\|_F^2 :\|Z^{12}\|_\infty\le \alpha  \Big\}\\
+&\min_{Z^{11}}\Big\{ \frac{\rho}{2} \|Z^{11}_{k \ell} -C^{11}_{k \ell} \|_F^2 : k\neq \ell \Big\}+ \min_{Z^{22}}\Big\{ \frac{\rho}{2} \|Z^{22}_{k \ell} -C^{11}_{k \ell} \|_F^2 : k\neq \ell \Big\}\\
+&\min_{\diag(Z)}\Big\{\lambda \|\diag(Z)\|_\infty+ \frac{\rho}{2} \|\diag(C)-\diag(Z) \|_2^2 \Big\}.
\end{aligned}
\end{equation}

For the first term of the above optimization problem, utilizing its separable structure, it is equivalent to
$$
\sum_{(j,k)\in S} \min_{|Z^{12}_{k \ell}|\le \alpha} \Big\{   \frac{1}{2}  \big(Z_{k \ell}^{12} - Y_{k \ell}\big)^2 + \rho\big(Z_{k \ell}^{12} - C_{k \ell}^{12}\big)^2\Big\} + \rho\sum_{(j,k)\not\in S} \min_{|Z_{k \ell}^{12}|\le \alpha}  \big(Z_{k \ell}^{12}-C_{k \ell}^{12}\big)^2 ,
$$
from which we see that its minimizer is given by $\cZ^{12}(C)$.

 In addition, the optimality of $\cZ^{11}(C)$ and $\cZ^{22}(C)$ are obtained by considering the %rest 
{remaining}
terms of \eqref{eqn:equiv}, which concludes the proof.
\end{proof}

Note that in \eqref{eqn:equiv}, we need to solve the following optimization problem
\begin{equation}\label{eqn:sub}
\min_{z\in\RR^d} ~ \beta \|z\|_\infty + \frac{1}{2}\|c - z\|^2_2,
\end{equation}
where $c = (c_1, \ldots ,c_d)^T = \diag(C)$ and $\beta = \lambda/\rho$. 
A direct approach to solve this problem is {to reformulate} it into a linearly {constrained} quadratic programming problem. In the next lemma, we show that 
it actually admits a  closed-form solution.  
For  ease of presentation, we assume without loss of generality that $c_1\geq c_2\geq \ldots\geq c_d\geq 0$.

\begin{lem} \label{lem:sub}
Suppose that $c_1\geq c_2\geq \ldots\geq c_d\geq 0$. 
The solution to the optimization problem \eqref{eqn:sub} is of the form 
$$
z^* = (t^*,\dots,t^*,c_{k^*+1},\dots, c_d)^T,
$$
where $t^* = \frac{1}{k^*} \sum_{i=1}^{k^*} (c_i-\beta)$
and $k^*$ is the index such that 
$c_{k^*+1} <   \frac{1}{k^*} \left(\sum_{i=1}^{k^*} c_i-\beta\right) \leq c_{k^*}$. If no such $k^*$ exists, {then} $z^* = (t^*,\ldots,t^*)^T$, where $t^*=\frac{1}{d}\sum_{i=1}^d(c_i-\beta)$.
\end{lem}
\begin{proof} Let $z=(z_1,\ldots,z_d)^T$. By the assumption that 
$c_1\geq c_2\geq \ldots\geq c_d\geq 0$, one can prove by contradiction that 
the optimal solution to \eqref{eqn:sub} must satisfy the property that
$z_1\geq z_2\geq \ldots \geq z_d \geq 0$. It is clear that \eqref{eqn:sub} 
is equivalent to the following convex minimization problem:
\begin{eqnarray}
 \min_{z,t} \Big\{ \beta t + \frac{1}{2}\norm{c-z}^2 :   z_i\ge 0,\;  z_i - t \leq 0,\; i=1,\dots, d\Big\},
\end{eqnarray}
whose KKT optimality conditions are given by
\begin{eqnarray*}
\begin{array}{l}
 z- c - \mu + \hat{\mu}=0,  \\[5pt]
 \beta - \sum_{j=1}^d \hat{\mu}_j = 0, \\[5pt]
 \mu_j \geq 0, \;z_j \geq 0,\; \mu_jz_j =0, \; j=1,\dots,d,\\[5pt]
 \hat{\mu}_j \geq 0,\; z_j -t \leq 0, \;\hat{\mu}_j(z_j-t) =0, \; j=1,\dots,d.
 \end{array}
\end{eqnarray*}
Define
$$
 t_k = \frac{1}{k} \left(\sum_{i=1}^k c_i-\beta\right),\quad k=1,\dots,d.
$$
Let $k^*$ be the index such that 
$c_{k^*+1} < t_{k^*} \leq c_{k^*}$.
If no such $k^*$ exists, i.e., {$c_i< t_i$} for all $i=1,\dots,d$,  {then} set $k^*=d$.
{Now} one can verify that the point $(z^*,t^*,\mu^*,\hat{\mu}^*)$ defined below satisfies the
KKT conditions: 
\begin{eqnarray*}
& \mu^* = 0, \;\;t^* = t_{k^*},\;\;
%\\[5pt]
%& & 
z_i^* = \left\{ \begin{array}{ll}
 t^* &\mbox{for $i=1,\dots,k^*$} ,
\\[5pt]
 c_i  &\mbox{for $i=k^*+1,\dots,d$,} 
\end{array}\right. \;\;
\hat{\mu}_i^* = \left\{ \begin{array}{ll}
 c_i-t^* &\mbox{for $i=1,\dots,k^*$}, 
\\[5pt]
 0 &\mbox{for $i=k^*+1,\dots,d$.} 
\end{array}\right. 
\end{eqnarray*}
Hence $z^*$ is the optimal solution to \eqref{eqn:sub}. This completes the proof. 
\end{proof}

\begin{rmk}
We avoid presenting the general case of $c = (c_1, \ldots ,c_d)^T$ for {simplicity}. The solution in the general case can be derived similarly, and we implement the algorithm for the general case in  later numerical studies.
\end{rmk}

%The proof of Lemma \ref{lem:sub} is by deriving the optimality condition of 

%Lemma \ref{lem:sub} shows that the sub-optimization problem \eqref{eqn:sub} can be solved efficiently.

%Without loss of the generality, let us assume $\diag(C) = (c_1,...,c_d)^T\in\RR^d_+$ and $c_1\ge c_2\ge...\ge c_d \ge 0$. It is not difficult to see that the solution $z = (R_k,...,R_k,c_{k+1},...,c_d)^T$ for some $k$, where $R_k = (k+2\lambda\rho)^{-1}\sum_{i=1}^kc_{i} $. In addition, it is not difficult to see that the optimal $k$ is the first such $k$ that $2\lambda R_{k+1} + \rho\sum_{i=1}^{k+1}(R_{k+1}-c_i)\ge 0$, which we can find efficiently. Thus, solving the sub-optimization problem \eqref{eqn:sub} can be conducted efficiently.

The algorithm for solving problem \eqref{eqn:sdp} is summarized in Algorithm 1.

 \begin{algorithm}[htb]\label{alg:maxnorm}
\caption{Solving max-norm optimization problem (\ref{eqn:sdp}) by the ADMM}
%Initialize $X^0$, $Z^0$, $W^0$, $\rho$, $\lambda$.
\begin{algorithmic}
\REQUIRE $X^0$, $Z^0$, $W^0$, $Y_\Omega$, $\lambda$, $\mu$, $\alpha$, $\rho$, $\tau$, $t=0$.
\WHILE {Stopping criterion is not satisfied.}
\STATE Update $X^{t+1}\leftarrow\Pi_{\cS_+^{d}}\big\{Z^t -\rho^{-1}{(W^t+\mu I)}\big\}$.
\STATE Update $Z^{t+1}\leftarrow\cZ(X^{t+1} +\rho^{-1}W^t)$ by \eqref{eqn:ZC1}.
\STATE Update $W^{t+1}\leftarrow W^{t} + {\tau}\rho(X^{t+1}-Z^{t+1})$.
\STATE $t\leftarrow t+1$.
\ENDWHILE
\ENSURE $\hat{Z} = Z^{t}$, $\hat{M} = \hat{Z}^{12}\in\RR^{d_1\times d_2}$.
\end{algorithmic}
\end{algorithm}

\begin{rmk}
Taking a closer look at Algorithm 1, %we have that considering an equivalent form \eqref{eqn:XZ} of the original problem 
{we  see that the equivalent reformulation \eqref{eqn:XZ} of the original problem}
 \eqref{eqn:sdp}   brings us computational efficiency. {In particular, all sub-problems can be solved efficiently.} Among them,
the most computationally expensive step is the $X$-update step as we need to compute  
{an} eigenvalue decomposition of the matrix $Z^t -\rho^{-1}W^t$, which has the complexity
{of} $\cO(d^3)$. {Meanwhile, we point out that if a rank-$r$ solution to the $X$-subproblem is desired, the computational complexity can be reduced to $\cO(rd^2)$. }
\end{rmk}

\begin{rmk}
{Note that if the user requires an exact low rank solution, solving the $X$-subproblem can be further accelerated. In particular, we can apply the Eckart-Young Theorem and project the solution onto the nearest face for the target rank. See, for example, \cite{oliveira2015admm}, where this idea is applied to the SDP relaxation of the quadratic assignment problem with nonnegativity constraints added.}
\end{rmk}

In addition to the algorithm for the regularized max-norm minimization problem \eqref{eqn:maxnormp}, we also provide the algorithm for solving the constrained version \eqref{eqn:mhat} in Appendix \ref{app:con}. We focus our discussions on the regularized version since it is computationally more challenging.

%%%%%%%%%%%%%%%%%%%%%%%%
\subsection{Stopping Conditions}

In this section, we discuss the stopping conditions for Algorithm 1. 
Denote by $\delta_{\cal C}(\cdot)$ the indicator function over
a given set ${\cal C}$ such that $\delta_{\cal C}(x) = 0$ if $x\in {\cal C}$
and $\delta_{\cal C}(x) = \infty $ if $x\not\in {\cal C}$.
The optimality conditions for 
\eqref{eqn:XZ} {are given as follows:}
\begin{eqnarray}
  0 &\in &     \partial \delta_{\cS^d_+}(X)  + \mu I + W, \quad
  0 \in \partial  \delta_{\cP}(Z) +  \nabla \cL(Z)  - W, \quad
 X- Z =0,
\end{eqnarray}
where $W$ is the Lagrangian multiplier associated with the equality constraint
$X-Z=0$. {Here $\partial \delta_{\cS^d_+}(X)$ denotes the subdifferential 
of $\delta_{\cS^d_+}(\cdot)$ at $X$; similarly for $\partial  \delta_{\cP}(Z)$.}

By {the} 
optimality conditions of $X^{t+1}$ and $Z^{t+1}$ in \eqref{eqn:admm}, we have that
$$
0  \;\in\;  \partial \delta_{\cS^d_+}(X^{t+1}) + \rho \big\{X^{t+1} - Z^t + \rho^{-1}(W^t + \mu I)\big\}
$$
if and only if
$$
 \rho (Z^{t}-Z^{t+1}) + W^{t+1}-\widetilde{W}^{t+1}  
\;\in\;  \partial \delta_{\cS^d_+}(X^{t+1})  + \mu I 
+ W^{t+1},
$$
and 
$$
0 \;\in\;  \partial \delta_\cP(Z^{t+1}) + 
\nabla \cL (Z^{t+1}) + \rho (Z^{t+1}-X^{t+1}-\rho^{-1} W^t)
$$
if and only if 
$$
\widetilde{W}^{t+1} - W^{t+1}  \;\in\;  \partial \delta_\cP(Z^{t+1}) +  \nabla \cL (Z^{t+1}) 
- W^{t+1},
$$
%\begin{eqnarray*}
%& &  0  \;\in\;  \partial \delta_{\cS^d_+}(X^{t+1}) + \rho \big(X^{t+1} - Z^t + \rho^{-1}(W^t + \mu I)\big) \\
%&\Leftrightarrow &   \rho (Z^{t}-Z^{t+1}) + W^{t+1}-\widetilde{W}^{t+1}  \;\in\;  \partial \delta_{\cS^d_+}(X^{t+1})  + \mu I + W^{t+1}, \\[5pt]
%& & 0 \;\in\;  \partial \delta_\cP(Z^{t+1}) +  \nabla \cL (Z^{t+1}) + \rho (Z^{t+1}-X^{t+1}-\rho^{-1} W^t) \\
%&\Leftrightarrow& \widetilde{W}^{t+1} - W^{t+1}  \;\in\;  \partial \delta_\cP(Z^{t+1}) +  \nabla \cL (Z^{t+1}) - W^{t+1},\end{eqnarray*}
where $\widetilde{W}^{t+1} = W^t + \rho (X^{t+1}-Z^{t+1})$. 
Observe that the iterate $(X^{t+1},Z^{t+1},W^{t+1})$ generated from 
 Algorithm 1 is an accurate approximate optimal solution to 
\eqref{eqn:admm} if the residual
$$
\eta^{t+1} := \max\{ R_P^{t+1}, R_D^{t+1}  \} 
$$
is small, where 
$$
R_P^{t+1} = \norm{X^{t+1}-Z^{t+1}}_F,\quad
R_D^{t+1} = \max\big\{ \norm{\rho (Z^{t}-Z^{t+1}) + W^{t+1}-\widetilde{W}^{t+1} }_F, \; \norm{W^{t+1}-\widetilde{W}^{t+1}}_F\big\},
$$
denote the primal and dual residuals.
 In the practical implementation, we let the algorithm stop when
$\eta^{t+1} \leq 10^{-4}$ or when the number of iterations exceeds $200$.

%%%%%%%%%%%%%%%%%%%%%%
\subsection{Practical Implementations}

We should mention that tuning  the parameter $\rho$ properly in the ADMM 
method is critical for the method to 
%convergence at a fast rate.
{converge at a reasonable rate.}
In our implementation, starting with the initial value of 
$0.1$ for $\rho$,  we adaptively tune the parameter at every tenth iterations
based on the following criterion: 
\begin{eqnarray*}
 \left\{ \begin{array}{ll} 
\mbox{set } \rho \leftarrow 0.7\rho  &\mbox{if $\norm{R_P^{t+1}} < 0.5\norm{R_D^{t+1}}$},
\\[5pt]
\mbox{set } \rho \leftarrow 1.3\rho  &\mbox{if $\norm{R_D^{t+1}} < 0.5\norm{R_P^{t+1}}$.}
\end{array} \right.
\end{eqnarray*}
The basic idea is to balance the progress of $\norm{R_P^{t+1}}$ and 
$\norm{R_D^{t+1}}$ so that the stopping criterion $\eta^{t+1}\leq 10^{-4}$ can be 
attained within  a small number of iterations. 

Another important computational issue which we need to address is to cut down the 
cost of computing the full eigenvalue decomposition 
in the $X$-update step in Algorithm 1. 
Given a matrix $G\in \cS^d$, 
we observe that to compute the projection $\Pi_{\cS^d_+}(G)$, we  need only
the eigen-pairs corresponding to the positive eigenvalues of $G$. Thus in our implementation, 
we use the LAPACK subroutine {\tt dsyevx.f} to compute only a 
partial eigenvalue decomposition of $G$ if we know that the number
of positive eigenvalues of $G$ is substantially smaller than $d$, say less than 
 $10\%$ of $d$. Such a partial eigenvalue decomposition is typically 
cheaper than a full eigenvalue decomposition when the number 
of eigenvalues of interest is much smaller than the dimension $d$. 
For Algorithm 1, at the $(t+1)$-th iteration, we estimate the 
potential number of positive eigenvalues of $G^t:= Z^t-\rho^{-1}(W^t+\mu I)$ (and hence
the rank of $X^{t+1}$)
based on the rank of the previously computed iterate $X^t$.  
Such an estimation is usually  accurate when the sequence of iterates
$\{(X^t,Y^t,Z^t)\}$ starts to converge. During the initial phase of
Algorithm 1, we do not have a good estimate on the rank of $X^{t+1}$,
and we compute the projection based on the full eigenvalue decomposition
of $G^t$.

To further reduce the cost of computing
$X^{t+1}$ in Algorithm 1, we employ a heuristic strategy to truncate the 
small positive eigenvalues of $G^t$ to 0's. That is, if there is a
group of positive eigenvalues of $G^t$ with magnitudes  which are significantly larger
than the remaining positive eigenvalues, we compute $X^{t+1}$ using only the 
eigen-pairs corresponding to the large positive eigenvalues of $G^t$. 
Such a strategy can significantly reduce the cost of computing $X^{t+1}$ since
the number of large eigenvalues of $G^t$ is typically  small 
in a low-rank matrix completion problem. A surprising bonus of
adopting such a cost cutting heuristic is that the recovery error can
actually become 30--50\% smaller, despite the fact that 
the computed $X^{t+1}$ now is only an approximate solution of the
$X$-update subproblem in Algorithm 1. One possible explanation for 
such a phenomenon is that the truncation of small positive eigenvalues
of $X^{t+1}$ to 0's actually has  a debiasing effect to eliminate the attenuation
of the singular values of the recovered matrix due to the presence of the 
convex regularization term. In the case of compressed sensing, 
such a debiasing effect has been explained in \cite{SJWright}.

%To develop an estimation procedure of the exactly low-rank matrix $M^*$, which is robust against (unknown) sampling distributions, we consider the following hybrid estimator
%\begin{align}
%	\hat{M}_{{\rm HB}}   =  \argmin_{M \in \mathcal{K}(\alpha, R)} \hat{\mathcal{L}}_n(M;Y) + \lambda \| M\|_1 ,  \label{combined.estimator}
%\end{align}
%where $\lambda=\lambda_n >0 $ is a regularization parameter, $\alpha$ and $R$ are user-specified parameters.

%%%%%%%%%%%%%%%%%%%%%%%%%%%%%%%%%%%%%%
\section{Theoretical Properties} \label{sec:thm}
In this section, we provide theoretical guarantees for the hybrid estimator \eqref{eqn:hybrid}. To facilitate our discussions, we introduce the following notations.
Let $X_1,\ldots, X_n$ be i.i.d. copies of a random matrix $X$ with distribution $\Pi=(\pi_{k \ell})_{k\in [d_1], \ell \in [d_2]}$ on the set $\mathcal{X} = \left\{  e_k(d_1) {e_\ell(d_2)^T} , k=1,\ldots, d_1, \ell = 1, \ldots, d_2  \right\}$, i.e., $\p\{X=e_k(d_1) e_\ell(d_2)^T \} = \pi_{k \ell}$, where $e_k(d)$ are the canonical basis vectors in $\brr^d$. By definition, 
\begin{align}
	\| A \|_{L_2(\Pi)}^2 = \e \langle A, X \rangle^2  = \sum_{k=1}^{d_1} \sum_{\ell=1}^{d_2} \pi_{k \ell} A_{k \ell}^2
\end{align}
for all matrices $A= (A_{k\ell})_{1\leq k\leq d_1, 1\leq \ell \leq d_2}\in \brr^{d_1 \times d_2}$. Moreover, let
$$
	\pi_{k\cdot } = \sum_{\ell=1}^{d_2} \pi_{k \ell} \ \ \mbox{ and } \ \  \pi_{ \cdot  \ell} = \sum_{k=1}^{d_1} \pi_{k \ell}  
$$
be, respectively, the probabilities of observing an element from the $k$-th row and the $\ell$-th column.

Considering the exact low-rank matrix recovery, i.e., $\text{rank}(M^0) \le r_0$, the first part of the next theorem shows that the estimator \eqref{eqn:hybrid} achieves a fast rate of convergence under some ``ideal" situations, and the second part indicates that it is also robust against non-uniform sampling schemes. For ease of presentation, we conduct the  analysis by considering a constrained form of \eqref{eqn:hybrid}, {namely},
%that
\begin{equation} \label{eqn:hybrid2}
\hat{M}: = \argmin_{M\in\cK(\alpha,R)} \frac{1}{n}\sum_{t=1}^n \big( Y_{i_t,j_t} - M_{i_t,j_t}  \big)^2  + \mu \|M\|_*, 
\end{equation}
{where } $\cK(\alpha, R) = \{M \in \brr^{d_1 \times d_2} : \|M\|_\infty \le \alpha\text{ and }\|M\|_{\max}\le R\}$. {Our proof partly follows the arguments in \cite{cai2013matrix}. The major technical challenge here is to carefully balance the tuning parameters $R$ and $\mu$ in \eqref{eqn:hybrid2} to achieve the desired recovery results for both the uniform and non-uniform sampling schemes. }
%obtains a fast rate of convergence under the ``ideal" situation, and is robust against unknown sampling schemes. 

\begin{thm}  \label{thm4.1}
Assume that $\| M^0\|_\infty \leq \alpha $, $\mbox{rank}\,(M^0) \leq r_0$ and that $\xi_1,\ldots, \xi_n$ are i.i.d. $N(0,1)$ random variables. The sampling distribution $\Pi$ is such that $\min_{(k,\ell)\in [d_1]\times [d_2]}\pi_{k\ell} \geq (\nu d_1 d_2)^{-1}$ for some $\nu \geq 1$. Choose $R \geq \alpha r_0^{1/2}$ in \eqref{eqn:hybrid2} and write $d=d_1 + d_2$.

\begin{itemize}
\item[(i)] Let $\mu = c_1 (dn)^{-1/2}$ for some constant $c_1>0$. Then, for a sample size $2 < n \leq  d_1 d_2$, the estimator $\hat{M}$ given at \eqref{eqn:hybrid2} satisfies
\begin{align}
 \frac{1}{d_1 d_2} \| \hat{M}  - M^0 \|_F^2 \lesssim \max  \left\{ \nu^2 \frac{r_0 d}{n} + \nu(\alpha \vee \sigma) \sqrt{ \frac{R^2 d}{n} },\, \nu \alpha^2 \frac{\log d}{n} \right\} \label{rate.2}
\end{align}
with probability at least $1-3d^{-1}$.

\item[(ii)] Let $\mu = c_2 \, \sigma \max_{(k,\ell) \in [d_1] \times [d_2]} (  \pi_{k\cdot} \vee \pi_{\cdot \ell } )^{1/2} $ for some constant $c_2>0$. Then, for a sample size $2 < n \leq  d_1 d_2$, the estimator $\hat{M}$ given at \eqref{eqn:hybrid2} satisfies
\begin{align}
 & \frac{1}{d_1 d_2}\| \hat{M}  - M^0 \|_F^2 \nonumber \\
 &  \lesssim \max  \left\{   \nu^2 ( \alpha \vee \sigma)^2  \max_{k,\ell} (\pi_{k\cdot} \vee \pi_{\cdot \ell}) \frac{r_0 d_1 d_2 \log d }{n},\, \nu \alpha^2 \sqrt{\frac{\log d}{n} } \right\}  \label{rate.1}
\end{align}
with probability at least $1-3d^{-1}$.
\end{itemize}
\end{thm}
\begin{proof}[Proof of Theorem~\ref{thm4.1}]
{Recall that $Y_{i_t,j_t} = M^0_{i_t,j_t} + \sigma \xi_t = 
\inprod{X_t}{M^0} + \sigma\xi_t$ for $t=1,\ldots,n$.}
By the {optimality} of $\hat{M}$ in \eqref{eqn:hybrid2}, {we have that}
\begin{align}
	\frac{1}{n} \sn  \langle X_i, \hat{M} - M^0 \rangle^2 \leq \frac{2\sigma}{n} \sn \xi_i \langle X_i, \hat{M}  - M^0 \rangle + \mu ( \| M^0 \|_* -  \| \hat{M} \|_* ) .   \label{ineq.1}
\end{align}

For each matrix $A\in \brr^{d_1\times d_2}$, denote by $u_j(A)$ and $v_j(A)$ the left and right orthonormal singular vectors of $A$, i.e., $A=\sum_{j=1}^{r} \sigma_j(A) u_j(A) v^T_j(A)$, where $r={\rm rank}(A)$ and $\sigma_1(A) \geq \cdots \geq \sigma_r(A)>0$ are the singular values of $A$. Let $S_1(A)$ and $S_2(A)$ be, respectively, the linear span of $\{u_j(A)\}$ and $\{v_j(A)\}$. Consequently, following the proof of Theorem~3 in \cite{klopp2014noisy} we have
\begin{align}
	\| M^0 \|_* -  \| \hat{M} \|_* \leq \| \mathcal{P}_{M^0}(\hat{M} - M^0) \|_*  - \| \mathcal{P}_{M^0}^{\bot} (\hat{M} - M^0) \|_* ,   \label{ineq.2}
\end{align}
where $\mathcal{P}_A(B) = P_{S_1(A)} B + P_{S_1(A)}^{\bot}  B P_{S_2(A)} $, 
{$\mathcal{P}_A^\perp(B) = P_{S_1(A)}^{\bot} B P_{S_2(A)}^{\bot}$}, 
and $P_S$ denotes the projector onto the linear  subspace $S$.

\medskip
\noindent

(i) Looking at the inequality \eqref{ineq.1}, it follows from (6.6) in \cite{cai2013matrix} that with probability greater than $1-d^{-1}$,
\begin{align}
\frac{1}{n} \sn  \langle X_i, \hat{M} - M^0 \rangle^2 \leq 24 \sigma  \sqrt{\frac{R^2 d}{n}} + \mu \sqrt{2r_0} \| \hat{M} - M^0 \|  , \nonumber
\end{align}
which, together with (6.13) of \cite{cai2013matrix}, implies that with probability at least $1-3d^{-1}$,
\begin{align}
   \frac{1}{2 \nu d_1 d_2 }\| \hat{M} - M^0  \|_F^2
  &  \leq  \frac{1}{2}\| \hat{M} - M^0  \|_{L_2(\Pi)}^2  \nonumber \\
  & \leq \max\left\{  \mu \sqrt{2 r_0} \| \hat{M} - M^0 \|_F  +  C_1 ( \alpha + \sigma) \sqrt{\frac{R^2 d}{n}}    ,\,   C_2  \alpha^2 \frac{\log d}{n}  \right\} \nonumber \\
  & \leq \max\left\{   \frac{1}{4\nu d_1 d_2} \| \hat{M} - M^0 \|_F^2 + 2 \nu r_0 d_1 d_2 \mu^2  +  C_1 ( \alpha + \sigma)\sqrt{\frac{R^2 d}{n} }   , \,  C_2  \alpha^2 \frac{\log d}{n}  \right\} , \nonumber
\end{align}
where $C_5, C_6>0$ are absolute constants. This proves \eqref{rate.2} by rearranging the constants.

(ii) First we assume that the regularization parameter $\mu$ satisfies $\mu \geq 3 \| \Sigma_\xi \|$, where $\Sigma_\xi : = n^{-1} \sn \xi_i X_i \in \brr^{d_1 \times d_2}$. By \eqref{ineq.2} and the inequality $|\langle A, B \rangle |\leq \| A \|  \cdot \| B \|_* $ {which holds for all matrices $A$ and $B$ where $\norm{A}$ is the spectral norm}, the right-hand side of \eqref{ineq.1} is bounded by
\begin{align}
	& ( 2\sigma \| \Sigma_\xi \| + \mu )   \| \mathcal{P}_{M^0} (\hat{M} - M^0) \|_*  + ( 2\sigma \| \Sigma_\xi \|  - \mu )  \| \mathcal{P}_{M^0}^{\bot} (\hat{M} - M^0) \|_*  \nonumber \\
	&  \leq \frac{5}{3} \mu \| \mathcal{P}_{M^0} (\hat{M} - M^0) \|_* \leq \frac{5}{3} \mu \sqrt{ 2 r_0 } \|  \hat{M} - M^0 \|_F 
\end{align}
whenever $\mu \geq 3\sigma \| \Sigma_\xi \|$, where $r_0={\rm rank}(M^0)$. 

Let $\varepsilon_1, \ldots, \varepsilon_n$ be i.i.d. Rademacher random variables. Then, it follows from Lemmas~12 and 13 in \cite{klopp2014noisy} that with probability greater than $1-2d^{-1}$,
\begin{align}
 & \frac{1}{2\nu d_1 d_2 } \| \hat{M} - M^0 \|_F^2  \nonumber \\
 & \leq  \frac{1}{2} \| \hat{M} - M^0 \|_{L_2(\Pi)}^2 \nonumber  \\
 &  \leq \max\left\{ \frac{5}{3} \mu    \sqrt{2 r_0} \| \hat{M} - M^0 \|_F + C_3 \nu \alpha^2 r_0 d_1 d_2 (\e \| \Sigma_{\varepsilon} \| )^2 ,   C_4 \alpha^2 \sqrt{ \frac{ \log d}{n} }\right\}  \nonumber \\
 & \leq \max\left[ \frac{1}{4\nu d_1 d_2 } \| \hat{M} - M^0 \|_F^2 + \nu r_0 d_1 d_2\left\{ 6   \mu^2   + C_3  \alpha^2  (\e \| \Sigma_{\varepsilon} \| )^2 \right\}  ,   C_4 \alpha^2 \sqrt{\frac{ \log d}{n} } \right]  , \nonumber
\end{align}
where $C_3, C_4>0$ are absolute constants and $\Sigma_\varepsilon : = n^{-1} \sn \varepsilon_i X_i$.

It remains to consider the quantities $\| \Sigma_\xi \| $ and $\e \| \Sigma_\varepsilon \| $. For $\| \Sigma_\xi\|$ with Gaussian multipliers $\xi_1,\ldots, \xi_n$, applying Lemma~5 in \cite{klopp2014noisy} yields that, for every {$n >0 $},
\begin{align}
	\| \Sigma_\xi \| \leq C_5 \max_{(k,\ell) \in [d_1] \times [d_2]} (  \pi_{k\cdot} \vee \pi_{\cdot \ell } )^{1/2} \sqrt{\frac{ \log d}{n} } +  C_6 \frac{ (\log d)^{3/2} }{n} \nonumber
\end{align}
holds with probability at least $1-d^{-1}$, where $C_5,C_6>0$ are absolute constants. Furthermore, by Corollary~8.2 in \cite{mackey2014matrix},
\begin{align}
 	\e \| \Sigma_{\varepsilon} \| \leq  \max_{(k,\ell) \in [d_1] \times [d_2]} (  \pi_{k\cdot} \vee \pi_{\cdot \ell } )^{1/2}  \sqrt{\frac{3\log d}{n} } + \frac{\log d }{n}. \nonumber
\end{align}
Together, the previous three displays prove \eqref{rate.1}.
\end{proof}

\begin{remark}  \label{rmk4.1}
{\rm
It is known that both the trace-norm and the max-norm serve as semidefinite relaxations of the rank. In the context of approximately low-rank matrix reconstruction, we consider two types of convex {relaxations} for low-rankness. For any $\alpha, R >0$, define {the} matrix classes
\begin{align}
	\mathcal{M}_{\max}(\alpha, R)  = \left\{ M \in \brr^{d_1\times d_2} : \| M \|_\infty \leq \alpha, \| M \|_{\max} \leq R \right\} \nonumber 
\end{align}
and
\begin{align}
	 \mathcal{M}_{{\rm tr}}(\alpha, R)  = \left\{ M \in \brr^{d_1\times d_2} : \| M \|_\infty \leq \alpha,  (d_1 d_2)^{-1/2}\| M \|_* \leq R \right\} . \nonumber
\end{align}
For any integer $1\leq r\leq \min(d_1, d_2)$, set $\mathcal{M}_{{\rm r}}(\alpha, r)  = \left\{ M \in \brr^{d_1\times d_2} : \| M \|_\infty \leq \alpha, \mbox{rank}\,(M) \leq r \right\}$ and note that $\mathcal{M}_{{\rm r}}(\alpha, r) \subset  \mathcal{M}_{\max}(\alpha, \alpha r^{1/2}) \subset \mathcal{M}_{{\rm tr}}(\alpha, \alpha r^{1/2})$. The following results \citep{cai2013matrix} provide recovery guarantees for approximately low-rank matrix completion in the sense that the target matrix $M^0$ either belongs to $\mathcal{M}_{\max}(\alpha, R)$ or $\mathcal{M}_{{\rm tr}}(\alpha, R)$ for some $\alpha, R>0$. As before, set $d=d_1 + d_2$.
\begin{itemize}
\item[(i)] Assume that $M^0 \in \mathcal{M}_{\max}(\alpha, R)$ and $\xi_1,\ldots, \xi_n$ are i.i.d. $N(0,1)$ random variables. Then, for a sample size $2<n\leq d_1 d_2$, the max-norm constrained least squares estimator $\hat{M}_{\max} := \argmin_{M \in \mathcal{K}(\alpha, R)} \frac{1}{n}\sum_{t=1}^n \big( Y_{i_t,j_t} - M_{i_t,j_t}  \big)^2$ satisfies 
\begin{align*}
\| \hat{M}_{\max} -  M^0  \|_{L_2(\Pi)}^2 \lesssim  (\alpha \vee \sigma ) \sqrt{ \frac{R^2 d}{n} } +   \frac{\alpha^2\log d}{n}  \nonumber
\end{align*}
with probability at least $1-3d^{-1} $.

\item[(ii)] Assume that $M^0 \in \mathcal{M}_{{\rm tr}}(\alpha, R)$, $\xi_1,\ldots, \xi_n$ are i.i.d. $N(0,1)$ random variables and that the sampling distribution $\Pi$ is uniform on $\mathcal{X}$. Then, for a sample size $2<n\leq d_1 d_2$, the trace-norm penalized estimator $\hat{M}_{{\rm tr}} := \argmin_{M : \| M \|_\infty \leq \alpha } \frac{1}{n}\sum_{t=1}^n \big( Y_{i_t,j_t} - M_{i_t,j_t}  \big)^2+ \mu \| M \|_*$ with $\mu \asymp \sigma \sqrt{ \frac{\log d}{d n} }$ satisfies
\begin{align}
	\frac{1}{d_1 d_2} \| \hat{M}_{{\rm tr}} -  M^0  \|_F^2 \lesssim  (\alpha \vee \sigma ) \sqrt{\frac{R^2 d \log d}{n} } + \frac{ \alpha^2 \log d}{n} \nonumber
\end{align}with probability at least $1-3d^{-1}$.

\end{itemize}}
\end{remark}

\begin{remark}  \label{rmk4.2}
{\rm
When the underlying matrix has exactly rank $r$, i.e., $M^0 \in \mathcal{M}_{{\rm r}}(\alpha,r)$, it is known that using the trace-norm regularized approach leads to a mean square error of order $\cO(  n^{-1} rd \log d )$. Under the uniform sampling scheme, the trace-norm regularized method is the most preferable one as it achieves optimal rate of convergence (up to a logarithmic factor) and is computationally attractive, although from a practical point of view, the uniform sampling assumption is controversial.  {In comparison with the result in \cite{cai2013matrix}, which is suboptimal under the uniform sampling scheme and exact low-rank setting, here  we established near optimal  recovery results (up to a logarithmic factor) under such a setting, and we can still guarantee recoveries under non-uniform sampling schemes.}

An important message we wish to convey is that, when learning in a non-uniform world, the underlying sampling distribution also contributes to the recovery guarantee. More specifically, Part (ii) of Theorem~\ref{thm4.1} sheds light on how  the sampling distribution affects the recovery error bound.  The optimal rate of convergence in the class of low-rank matrices is also achieved by $\hat{M}$ when the sampling scheme is  uniform. From \eqref{rate.2} and \eqref{rate.1}, we see that the actual performance of the hybrid estimator $\hat{M}$ depends heavily on the sampling distribution and so is the optimal choice of the regularization parameter $\mu$.
}
\end{remark}

%%%%%%%%%%%%%%%%%%%%%%%%%%%%%%%
\section{Numerical Experiments} \label{sec:sim}
We  compare 
{the} nuclear-norm, max-norm and hybrid regularizers for matrix completion on an iMac with Intel i5 Processor at 2.7GHz with 16GB memory. We test different methods on simulated and real datasets.  All the  tuning parameters are chosen by data splitting.

%\red{Under each setting,  we test each method with different tuning parameters, and report the best overall results and the corresponding parameters for each method. In practice, such parameters can be chosen by data-driven methods such as cross validation. }

\subsection{Simulated Datasets}
We first test the methods on simulated data, where we consider three sampling schemes. In the first scheme, the indices of observed entries are uniformly sampled  without replacement, while in the other two schemes, the indices are sampled non-uniformly.
Specifically, in  all  three schemes, we let the target matrix $M^0$ be generated by $M^0 = M_LM_R^T$, where  $M_L$ and $M_R$ are two $d_t\times r$ matrices, and each entry is sampled independently from a standard normal distribution $N(0,1)$. Thus, {$M^0\in \RR^{d_t\times d_t}$} is a rank $r$ matrix. In all  three settings, as listed in Tables \ref{tab:comadm1}, \ref{tab:comadm2} and \ref{tab:comadm3}. we consider different combinations of dimensions, ranks and sampling ratios (SR), where  {$\text{SR} = n/d_t^2$.}
{We compare the matrix recovery results using the nuclear-norm, max-norm penalized estimators and the hybrid estimator. For the nuclear-norm approach, we compute the estimator by adopting the accelerated proximal-gradient method discussed in \cite{toh2010accelerated}. For the max-norm and hybrid approaches, we compute the estimator by solving problem \eqref{eqn:sdp} using Algorithm \ref{alg:maxnorm}, where $\mu$ in \eqref{eqn:sdp} is set to 0 when we compute the max-norm penalized estimator.
}

In Scheme 1, we uniformly sample the entries. In Schemes 2 and 3, we conduct non-uniform sampling schemes in the following way. Denote by $\pi_{k\ell}$ the probability that the $(k,\ell)$-th entry is sampled. 
For each $(k,\ell) \in [d_t]\times [d_t]$, let $\pi_{k \ell} = p_kp_\ell$, where we let $p_k$ (and $p_\ell$) be
$$
p_k =
\begin{cases}
2p_0 \quad &\text{ if }k\le \frac{d_t}{10}\\
4p_0 \quad &\text{ if }\frac{d_t}{10}<k\le \frac{d_t}{5}\\
p_0 \quad &\text{ otherwise},
\end{cases}
\text{ for Scheme 2,\;\; and }\;\;
p_k =
\begin{cases}
3p_0 \quad &\text{ if }k\le \frac{d_t}{10}\\
9p_0 \quad &\text{ if }\frac{d_t}{10}<k\le \frac{d_t}{5}\\
p_0 \quad &\text{ otherwise},
\end{cases}
\text{ for Scheme 3},
$$
and $p_0$ is a normalizing constant such that $\sum_{k=1}^{d_t}p_k = 1$. %Also, we define $p_{\ell}$ similarly. %We sample $25\%\times d_t^2$ i.i.d indices from this distribution. 

%We first generate $25\%\times n^2$ i.i.d standard Gaussian random samples. Then, we rescale them to $[1,n^2]$ by an affine transformation, where we make the smallest sample to be 1 and the largest sample to be $n^2-2$. Then, we round all the samples. This sampling scheme makes the entries with indices near $(n/2,n/2)$ much more likely to be sampled. In addition, note this sampling scheme is actually ``with replacement" in the sense that many entries will be sampled more than once.

In the implementation of Algorithm 1, we set
the tuning parameter $\lambda$  {to be proportional to $\norm{Y_{\Omega}}_F$},
as suggested by \cite{toh2010accelerated},
{where $Y_{\Omega}$ denotes the partially observed matrix.
From the theoretical analysis in Section \ref{sec:thm},  we have that 
the parameter $\mu$ should  be smaller than $\lambda$ by a factor of about
${({d_1d_2}})^{-1/2}=d_t^{-1}$ in the hybrid approach. 
}

To evaluate the matrix recovery results, we adopt the metric of relative error (RE) 
{defined by}
$$
\text{RE} = \frac{\|\hat{M} - M^0\|_F}{\|M^0\|_F},
$$
where $\hat{M}$ is the output solution by the algorithm. We consider different settings of $d_t$, $r$ and SR. We run simulations under each setting  for five {different instances}. 
{We first consider the noiseless cases.
The averaged relative errors and running times are summarized in the upper halves of Tables \ref{tab:comadm1}, \ref{tab:comadm2} and \ref{tab:comadm3}, corresponding to Schemes 1, 2 and 3, respectively.} 
%To examine the robustness of each method, we use the same rule to choose tuning parameters in all three schemes.
In Table \ref{tab:comadm1}, where uniformly sampled data is considered, we find {that}
the nuclear-norm approach obtains the best %properties. 
recovery results. 
Meanwhile, we find {that} the hybrid approach performs significantly better than the 
{pure}
max-norm approach. This observation is 
%in line 
{consistent}
with the existing theoretical result that max-norm regularization does not perform as well as nuclear-norm {regularization}
if the observed entries are indeed uniformly sampled, and the proposed hybrid approach significantly boosts the performance of the max-norm regularized method without specifying data generating schemes.
In Tables \ref{tab:comadm2} and \ref{tab:comadm3}, where non-uniform sampling distributions are considered, we observe that both max-norm regularized and hybrid approaches significantly outperform the nuclear-norm approach, especially when the sampling ratio is low.
This observation matches the theoretical analysis in Section \ref{sec:thm} and
 \cite{cai2013matrix}.
 We also find that the hybrid approach always outperforms the max-norm approach. This is because, while the max-norm approach is robust, the additional nuclear norm penalization helps {to} fully utilize the underlying  low-rank %setting
{structure} in our generating schemes.

Next, we consider settings with noises, where we use the same sampling schemes as in Schemes 1, 2 and 3, and for each sampled entry, we observe a noisy sample:
$$
Y_{i_t,j_t}  = M^0_{i_t,j_t} + \sigma \xi_t\cdot \|M^0\|_\infty ,\text{ where }\sigma = 0.01 \text{ and }\xi_t\sim N(0,1).
$$
We report the averaged relative errors and running times in the lower halves of Tables  \ref{tab:comadm1}, \ref{tab:comadm2} and \ref{tab:comadm3}. As expected, under non-uniform sampling schemes,  the max-norm and hybrid approaches produce better recovery results than the nuclear-norm approach, and the hybrid approach outperforms the max-norm approach. Surprisingly, we find that under the uniform sampling scheme, the max-norm and hybrid approaches also outperform the nuclear-norm approach in the noisy setting. These observations provide  further  evidences that the max-norm and hybrid approaches are more robust to noises and sampling schemes than the nuclear-norm {approach} in practice.

{In addition, for the noisy setting, we plot how the relative errors decrease as sampling ratios increase under the three schemes. Specifically, for $\text{SR} = 0.08, 0.10,..., 0.22$, $r = 3, 5, 10$ and $d_t = 500$ and $1000$, we plot the averaged relative errors over five repetitions in Figures \ref{fig:s1}, \ref{fig:s2} and \ref{fig:s3}. Under the uniform sampling scheme, Figure \ref{fig:s1} shows that the nuclear-norm approach provides the best recovery results while the hybrid approach performs
much better than the max-norm approach.  Under non-uniform sampling schemes, Figures \ref{fig:s2} and \ref{fig:s3} demonstrate that the hybrid approach has the best performance, while the nuclear-norm approach gives the poorest results.
}

%As shown in Table \ref{tab:comadm3}, it can be clearly seen that under non-uniform sampling schemes, max-norm penalized estimator and the hybrid estimator leads to better results for matrix completion, which matches the theoretical analysis of Section \ref{sec:thm} and
 %\cite{cai2013matrix}.
 
  \begin{figure}[ht!]
 \centering
\subfigure{
\includegraphics[scale = 0.27]{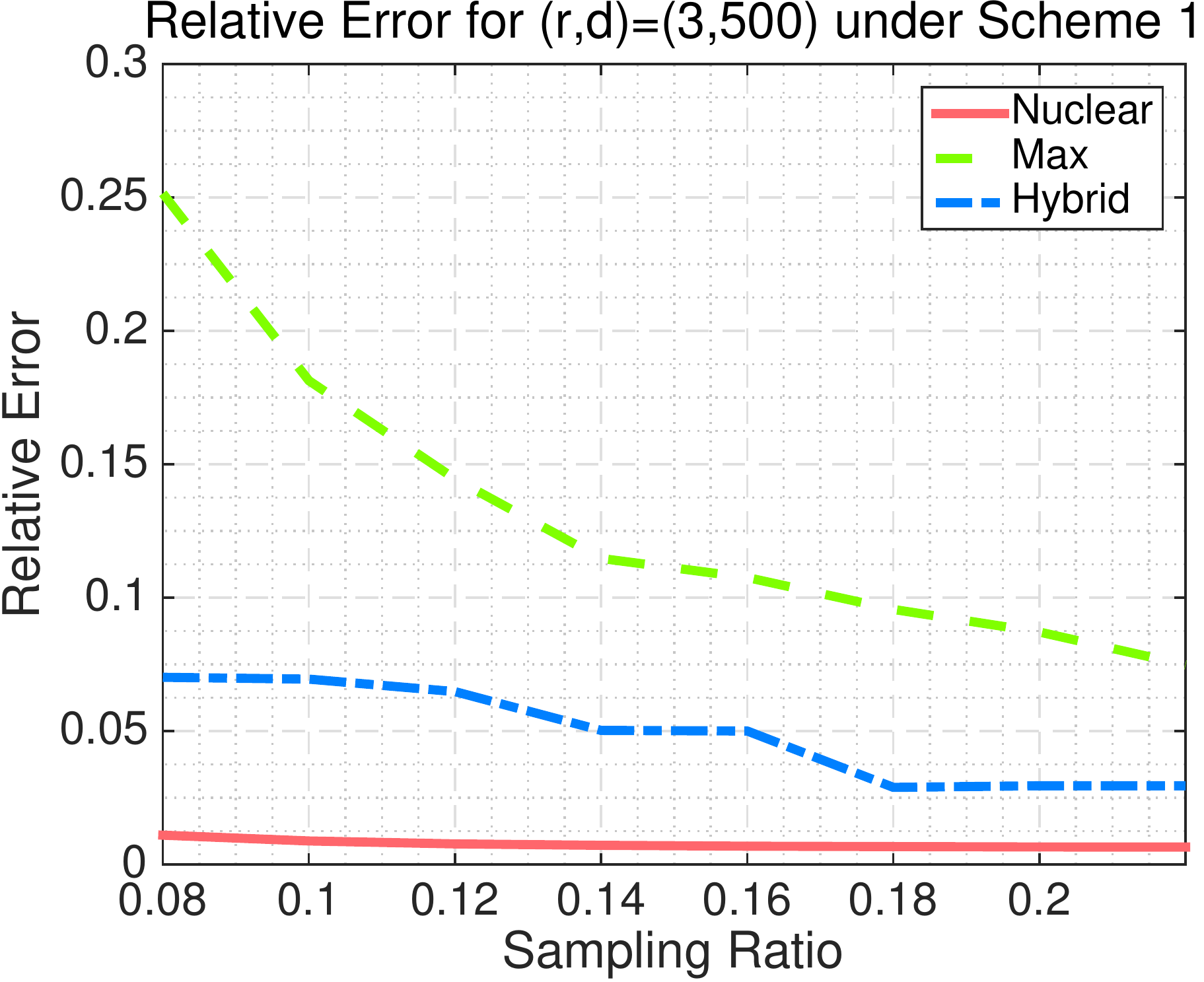}
}
\subfigure{
\includegraphics[scale = 0.27]{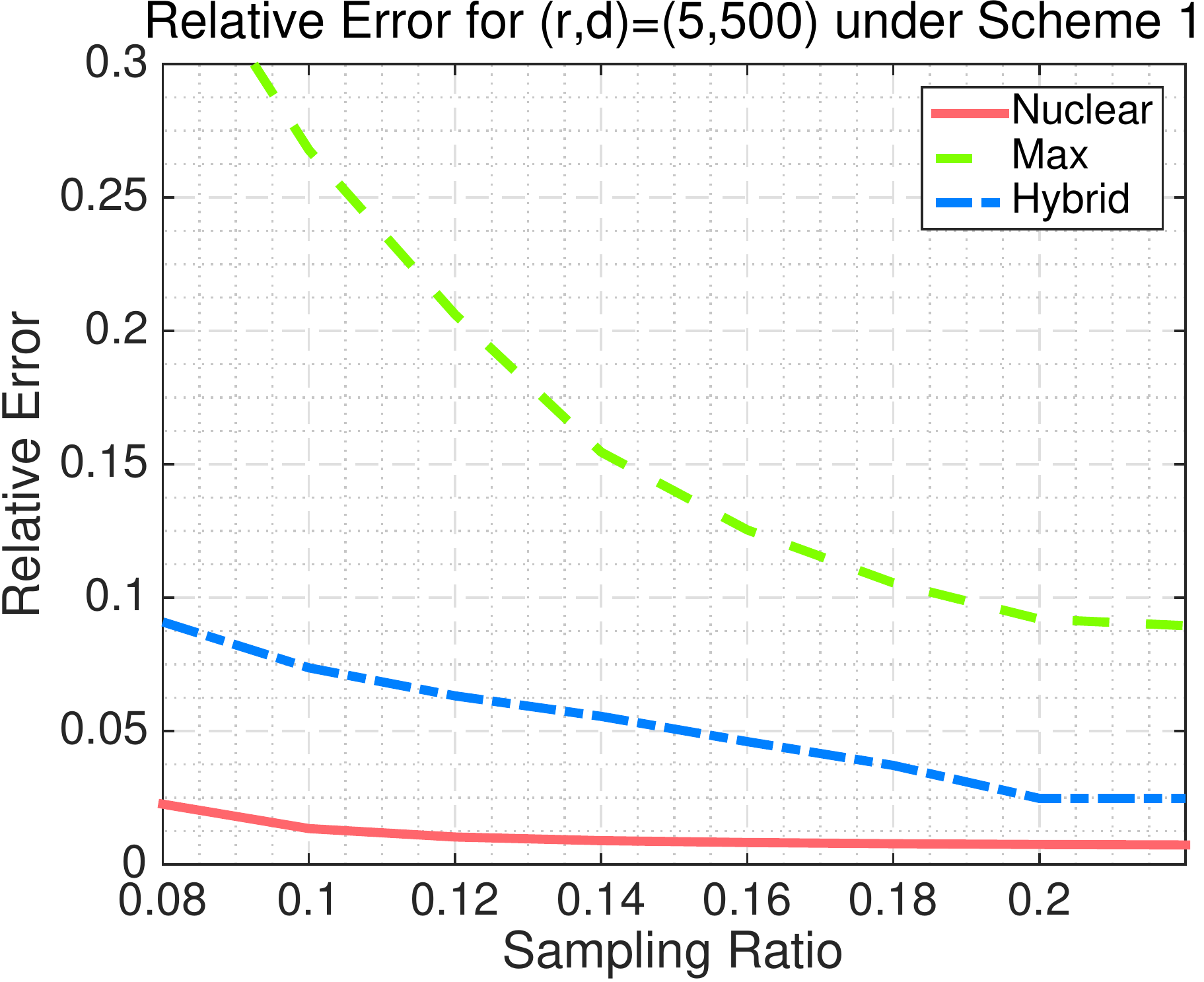}
}
\subfigure{
\includegraphics[scale = 0.27]{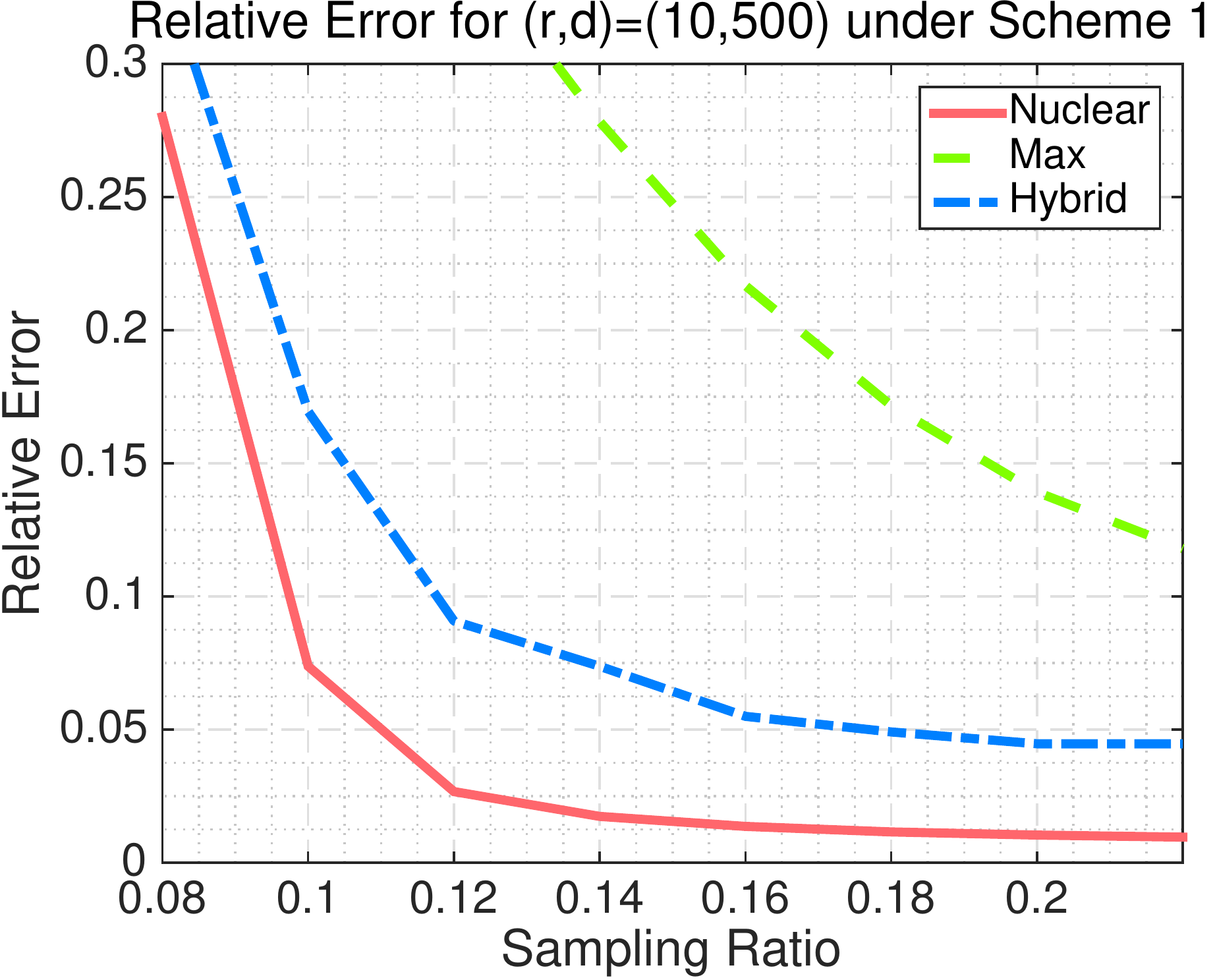}
}
\subfigure{
\includegraphics[scale = 0.27]{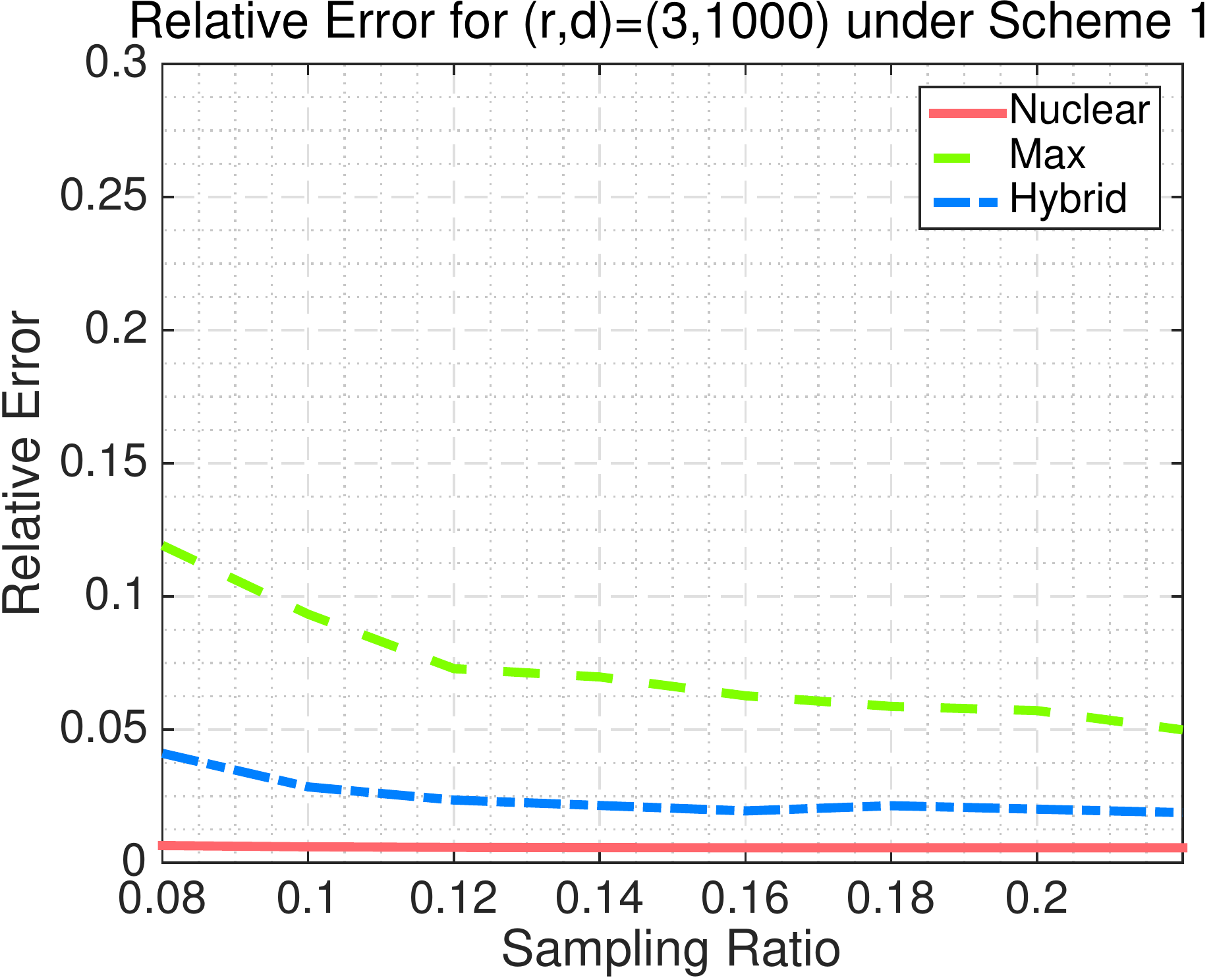}
}
\subfigure{
\includegraphics[scale = 0.27]{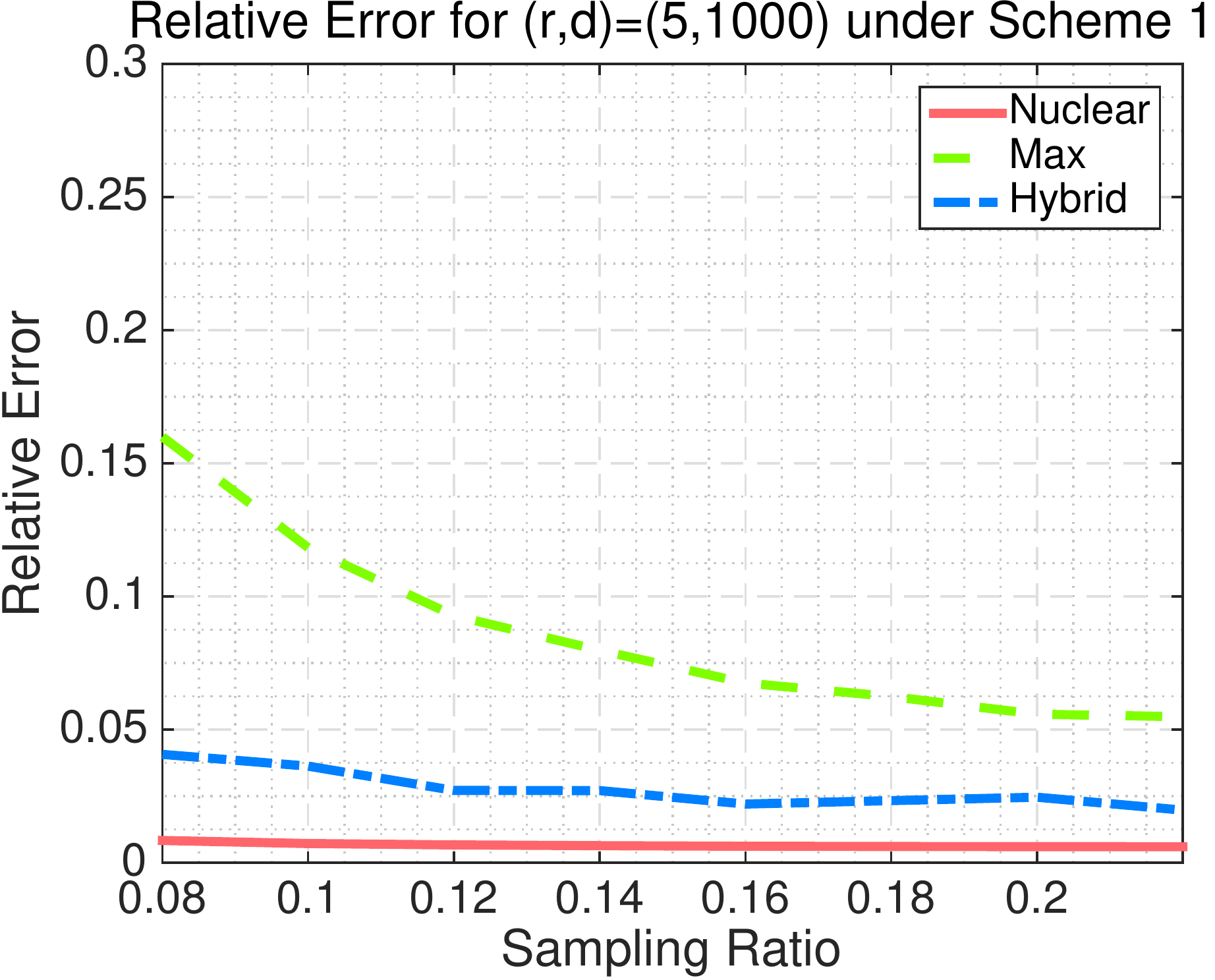}
}
\subfigure{
\includegraphics[scale = 0.27]{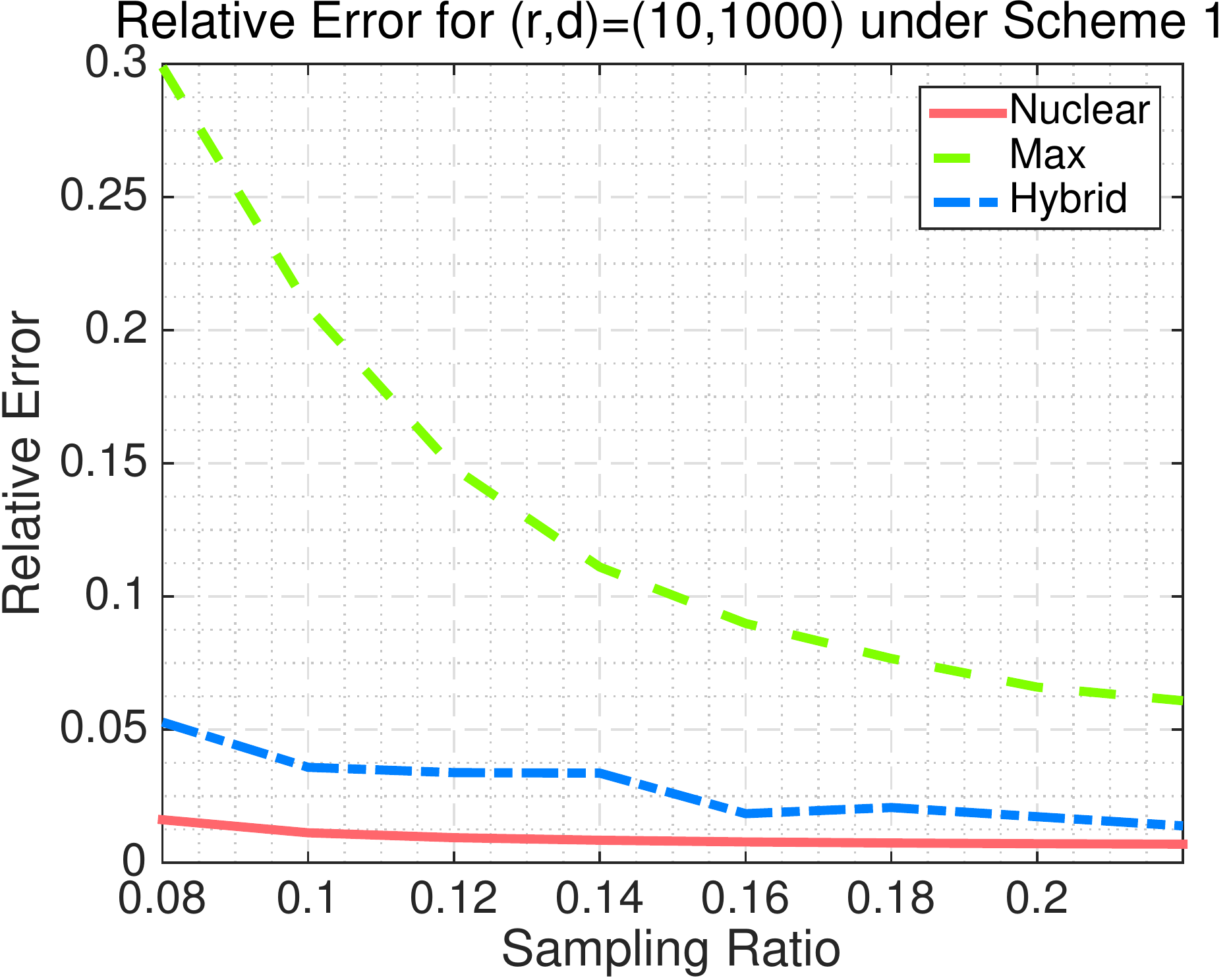}
}
\caption{Relative Errors under different settings for the noisy case under Scheme 1 (uniform sampling). }
\label{fig:s1}
\end{figure}

 \begin{figure}[ht!]
 \centering
\subfigure{
\includegraphics[scale = 0.27]{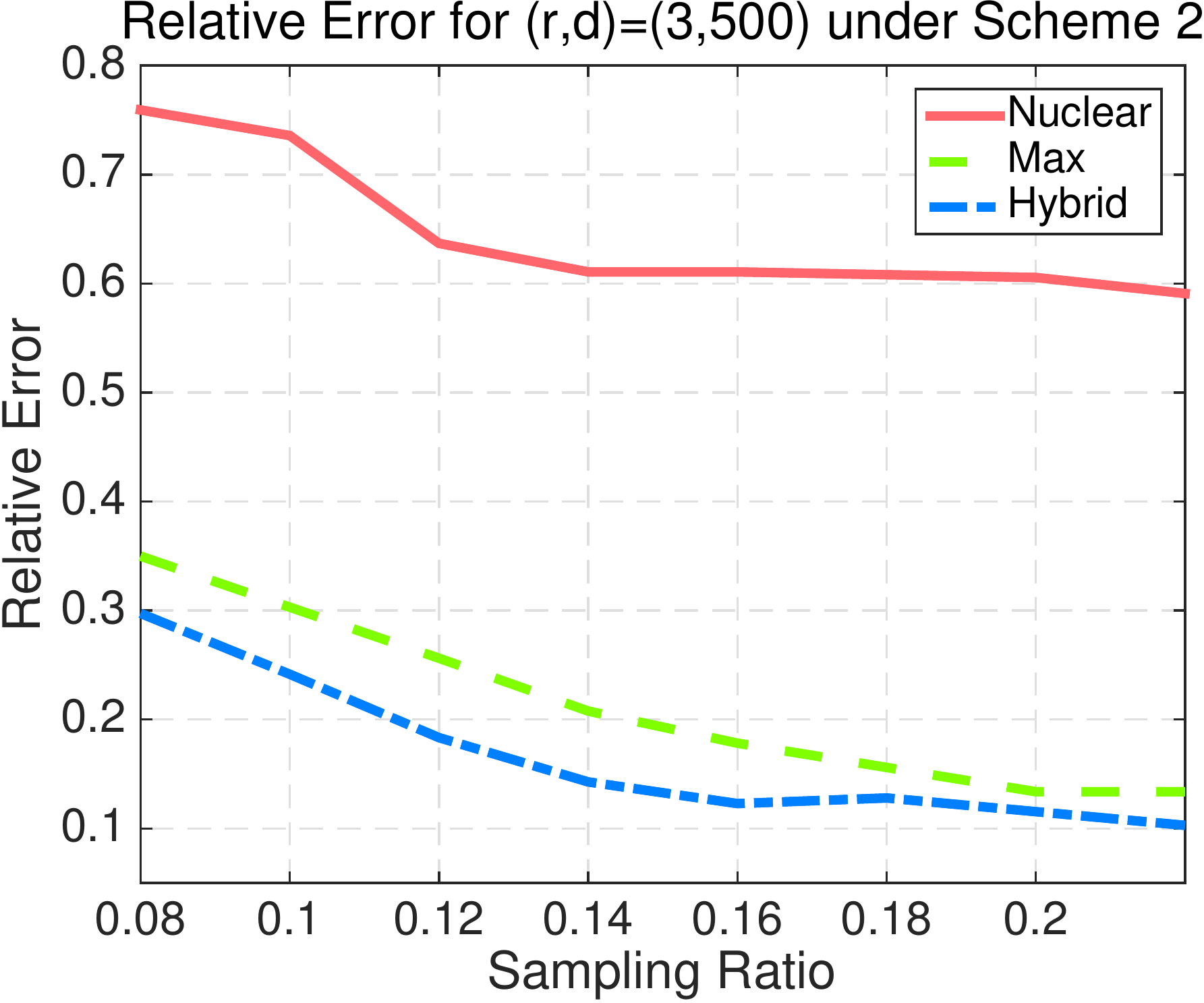}
}
\subfigure{
\includegraphics[scale = 0.27]{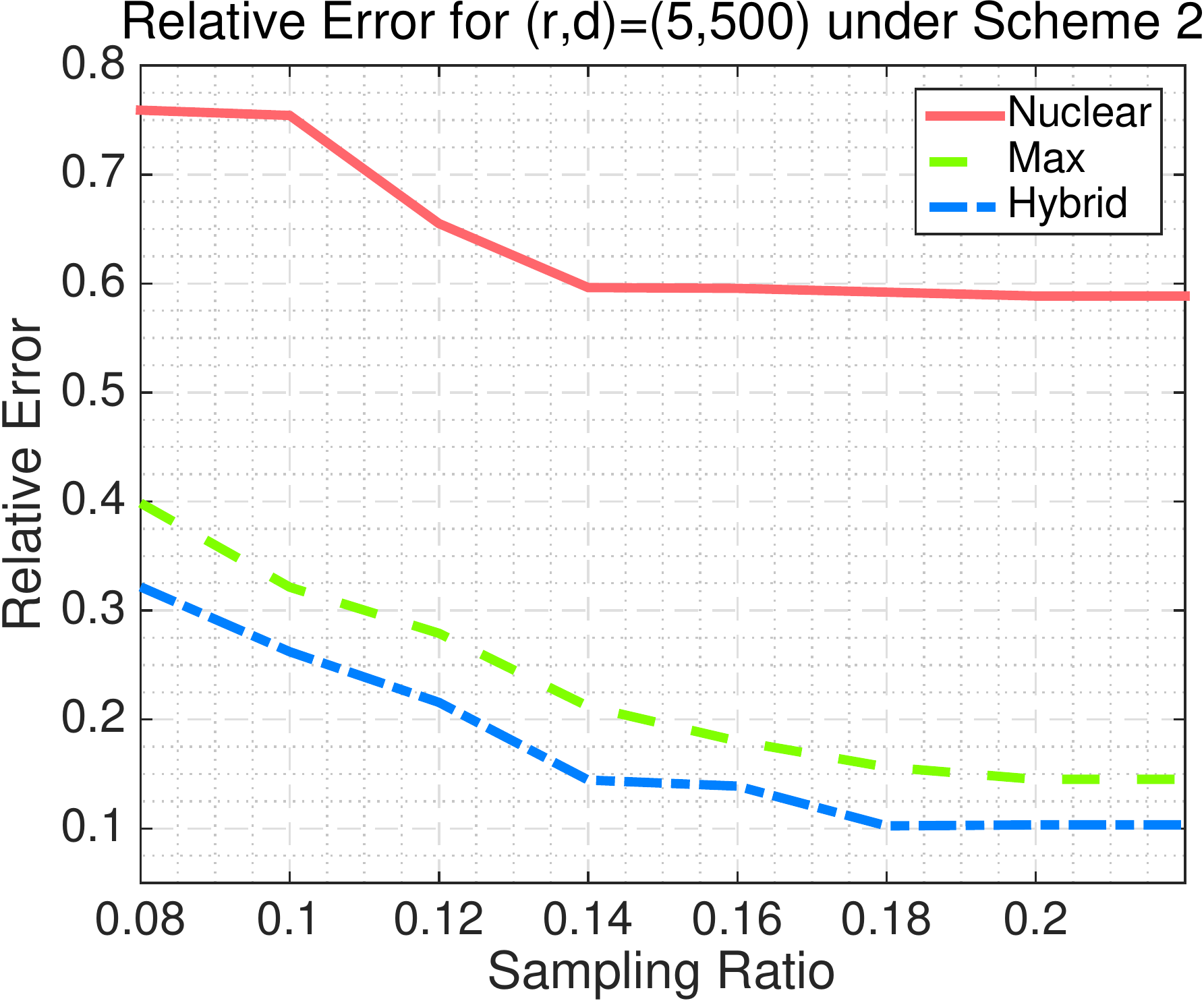}
}
\subfigure{
\includegraphics[scale = 0.27]{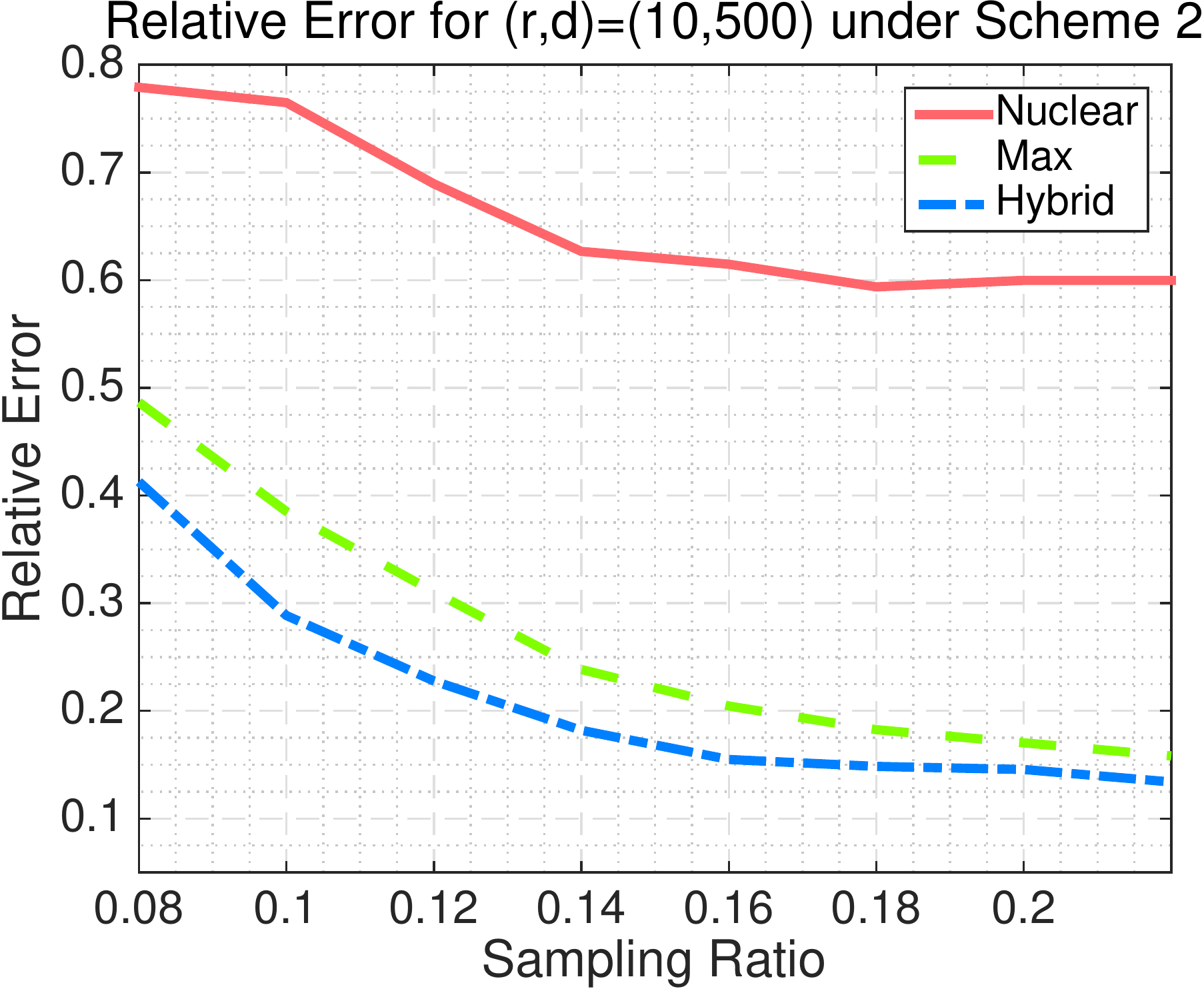}
}
\subfigure{
\includegraphics[scale = 0.27]{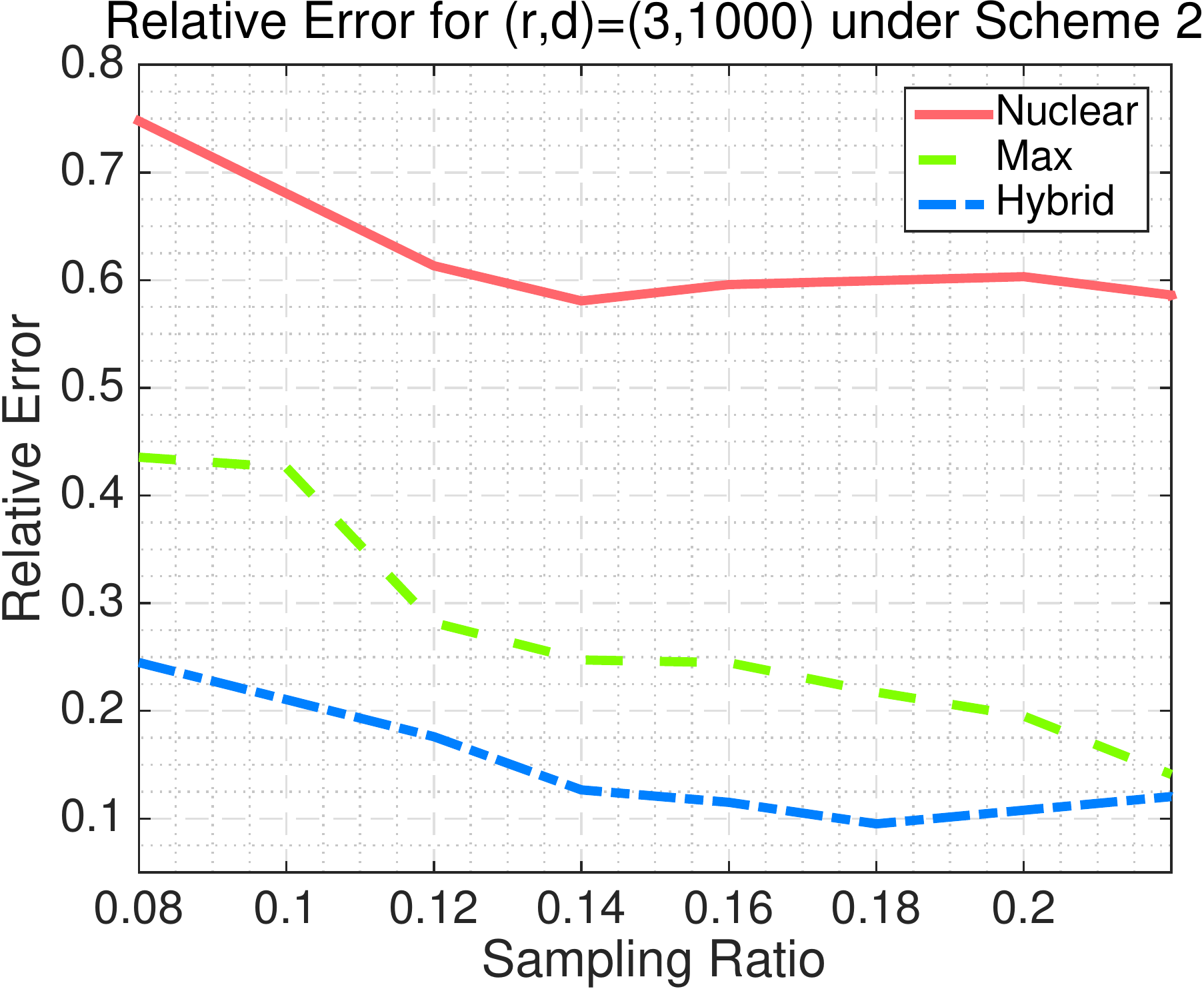}
}
\subfigure{
\includegraphics[scale = 0.27]{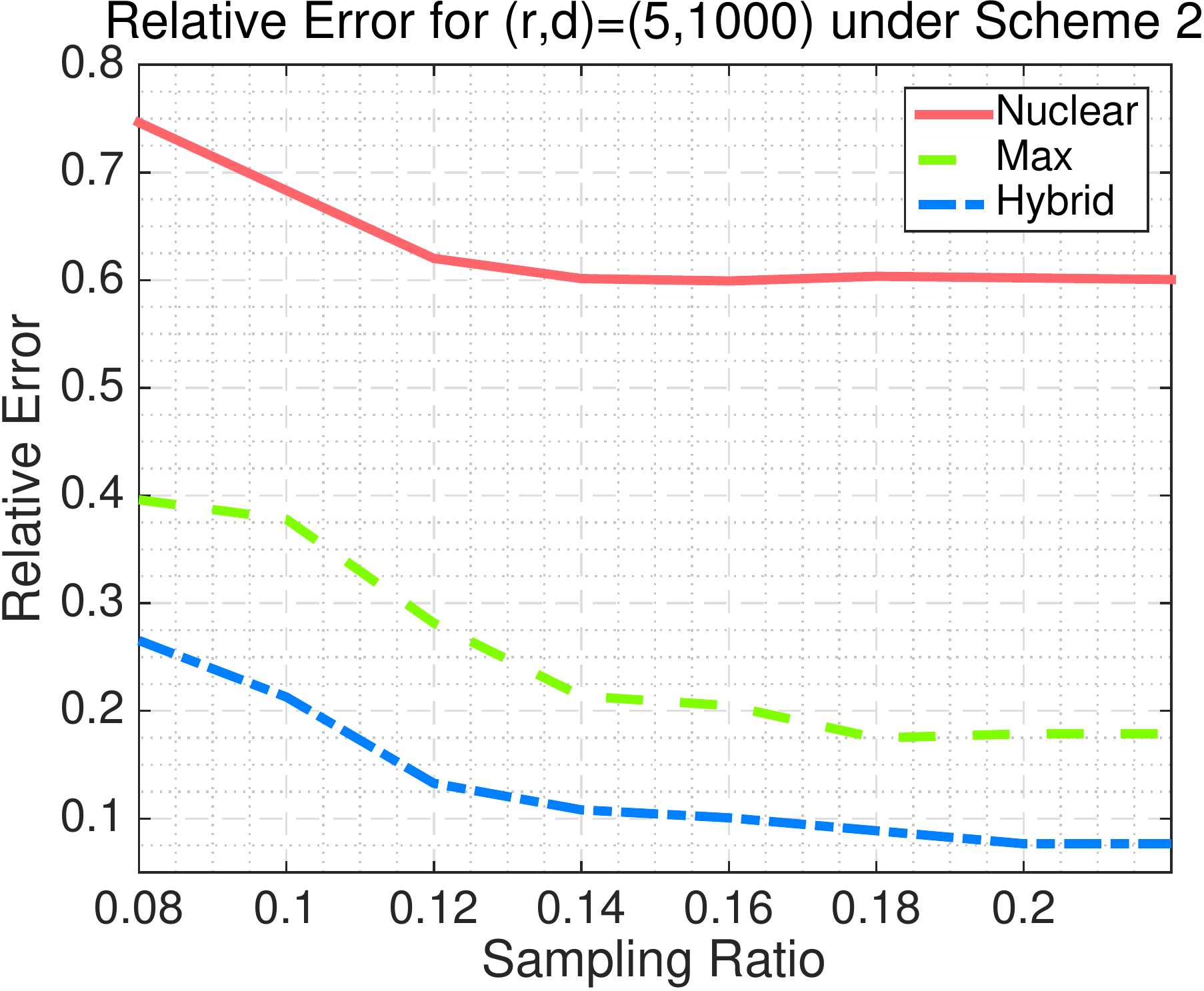}
}
\subfigure{
\includegraphics[scale = 0.27]{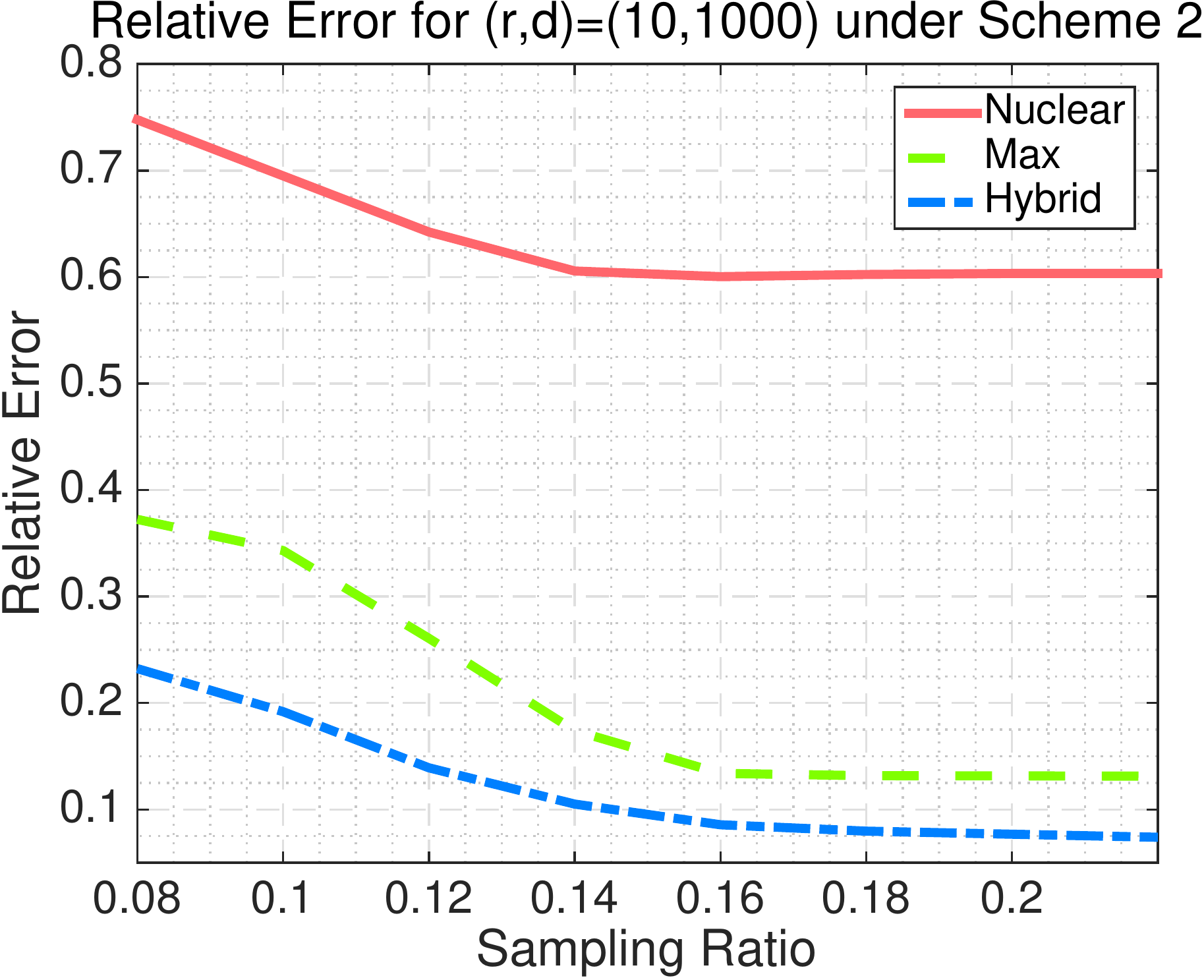}
}
\caption{Relative Errors under different settings for the noisy case under Scheme 2 (non-uniform sampling). }
\label{fig:s2}
\end{figure}

 \begin{figure}[ht!]
 \centering
\subfigure{
\includegraphics[scale = 0.27]{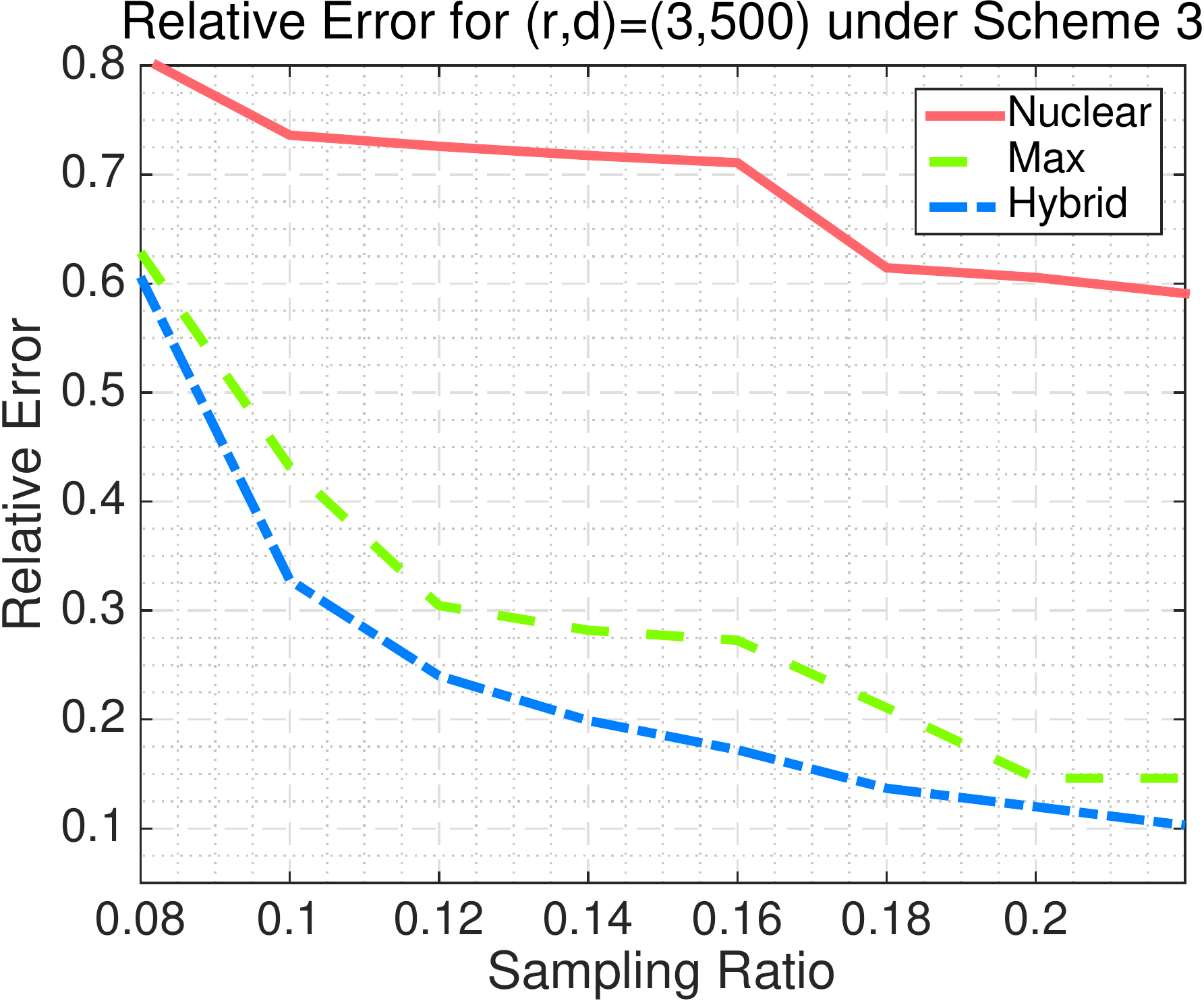}
}
\subfigure{
\includegraphics[scale = 0.27]{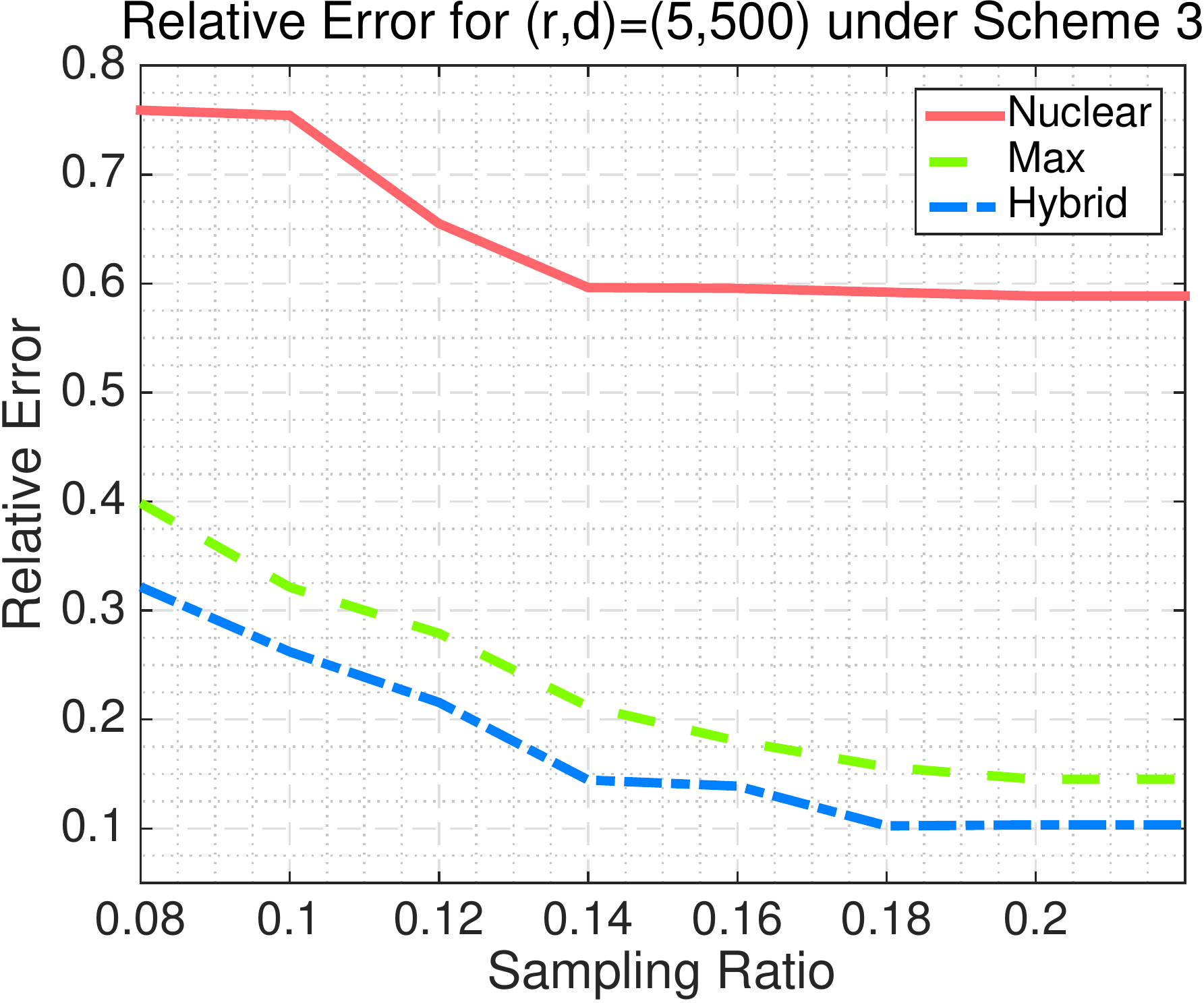}
}
\subfigure{
\includegraphics[scale = 0.27]{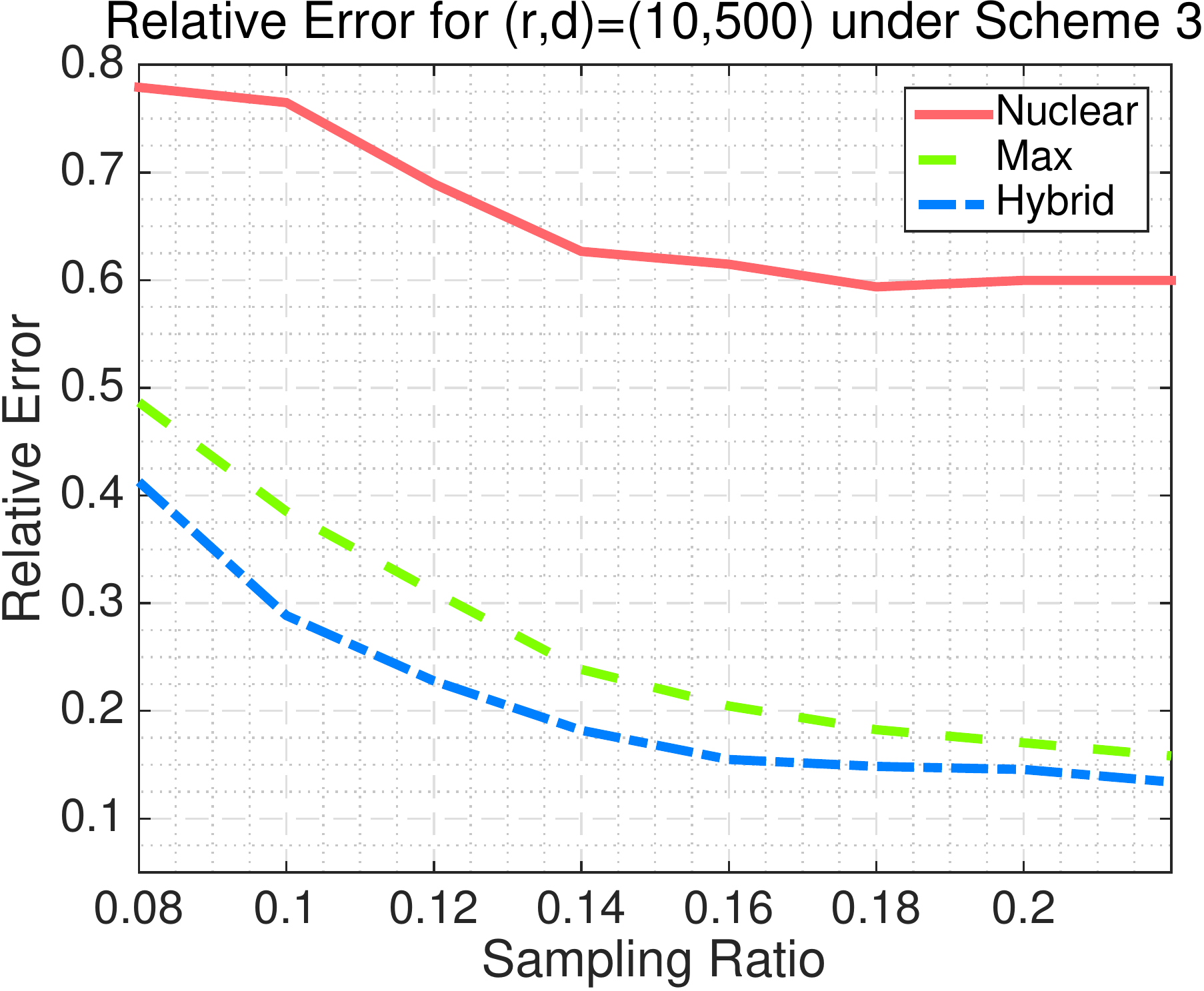}
}
\subfigure{
\includegraphics[scale = 0.27]{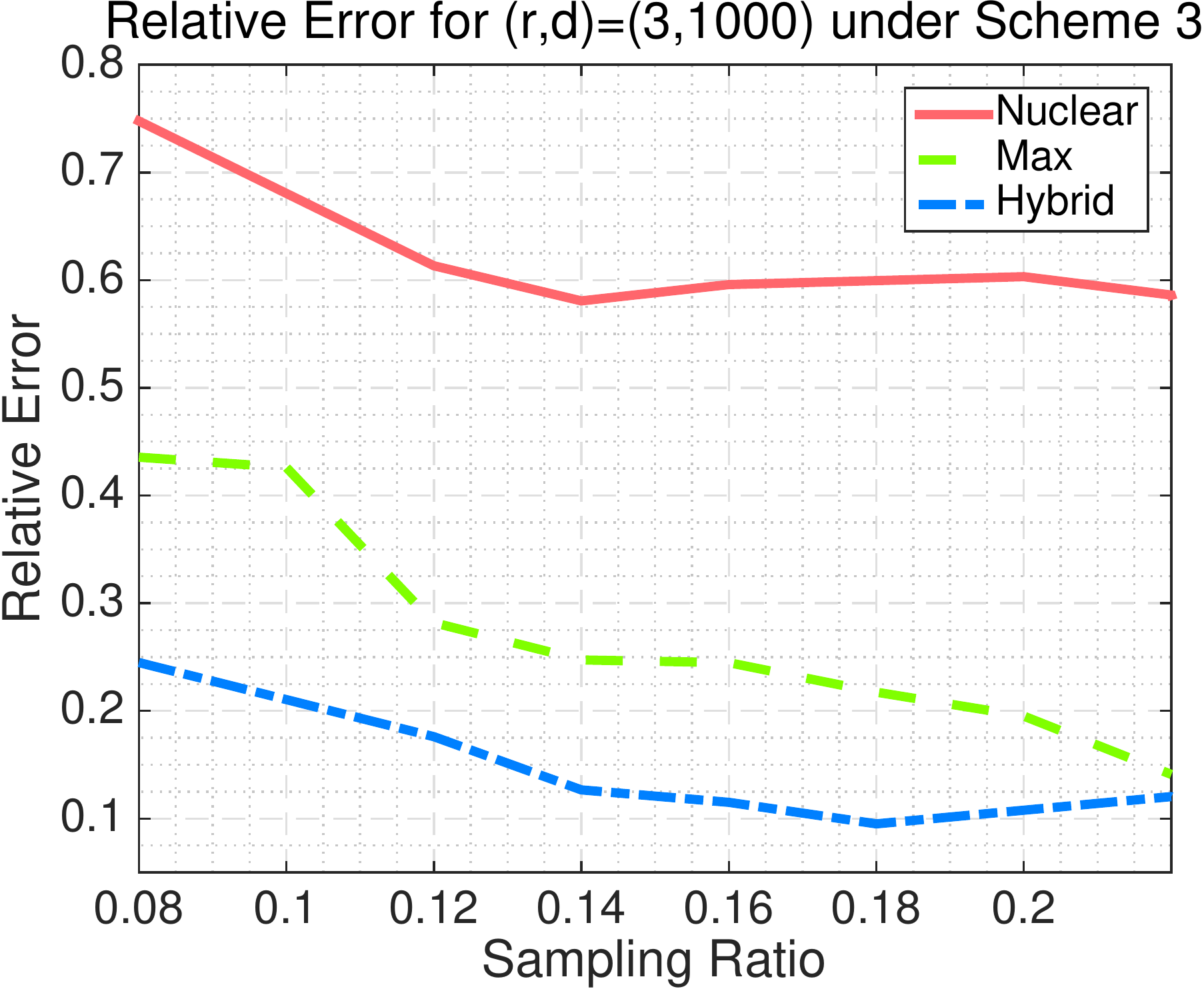}
}
\subfigure{
\includegraphics[scale = 0.27]{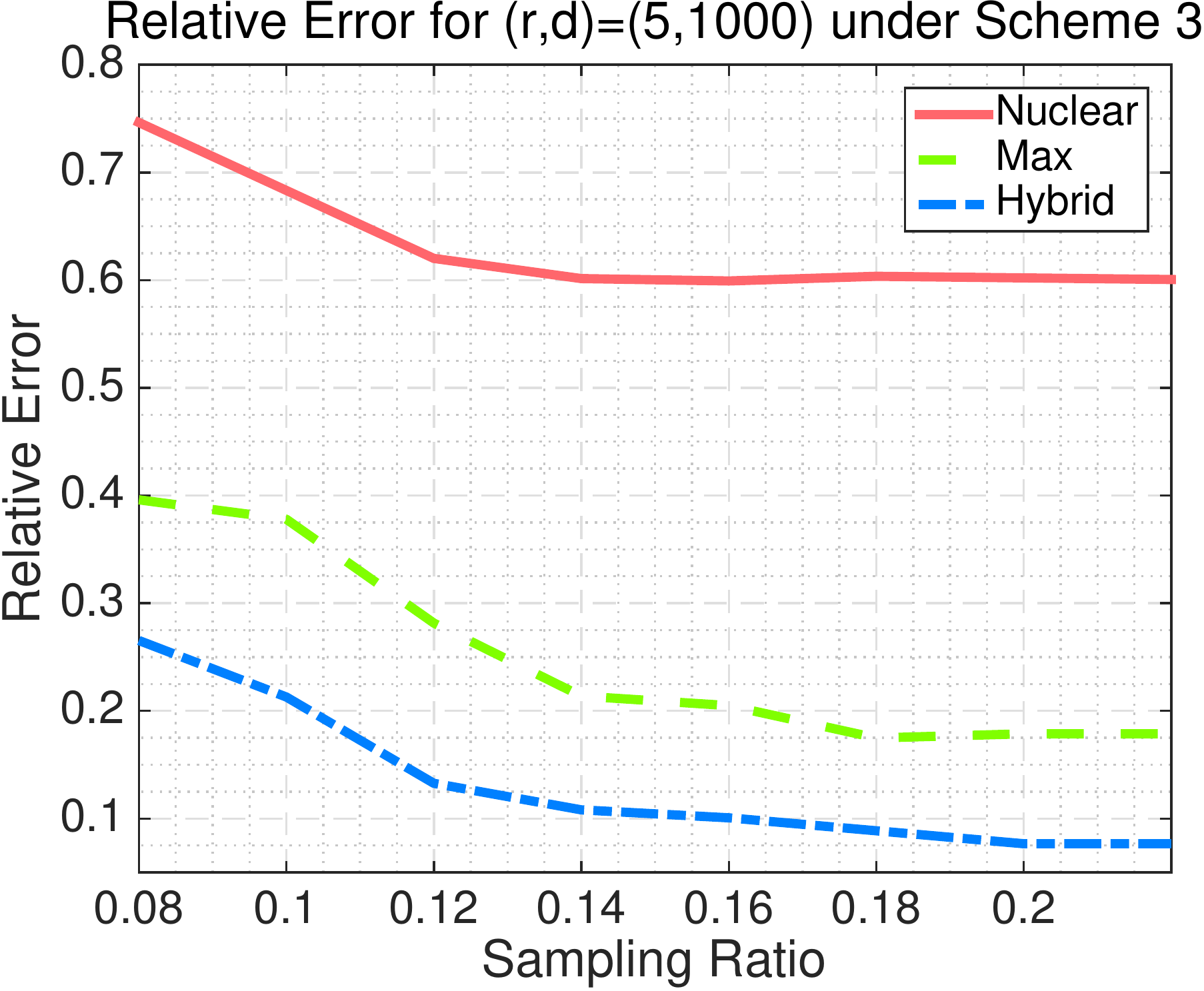}
}
\subfigure{
\includegraphics[scale = 0.27]{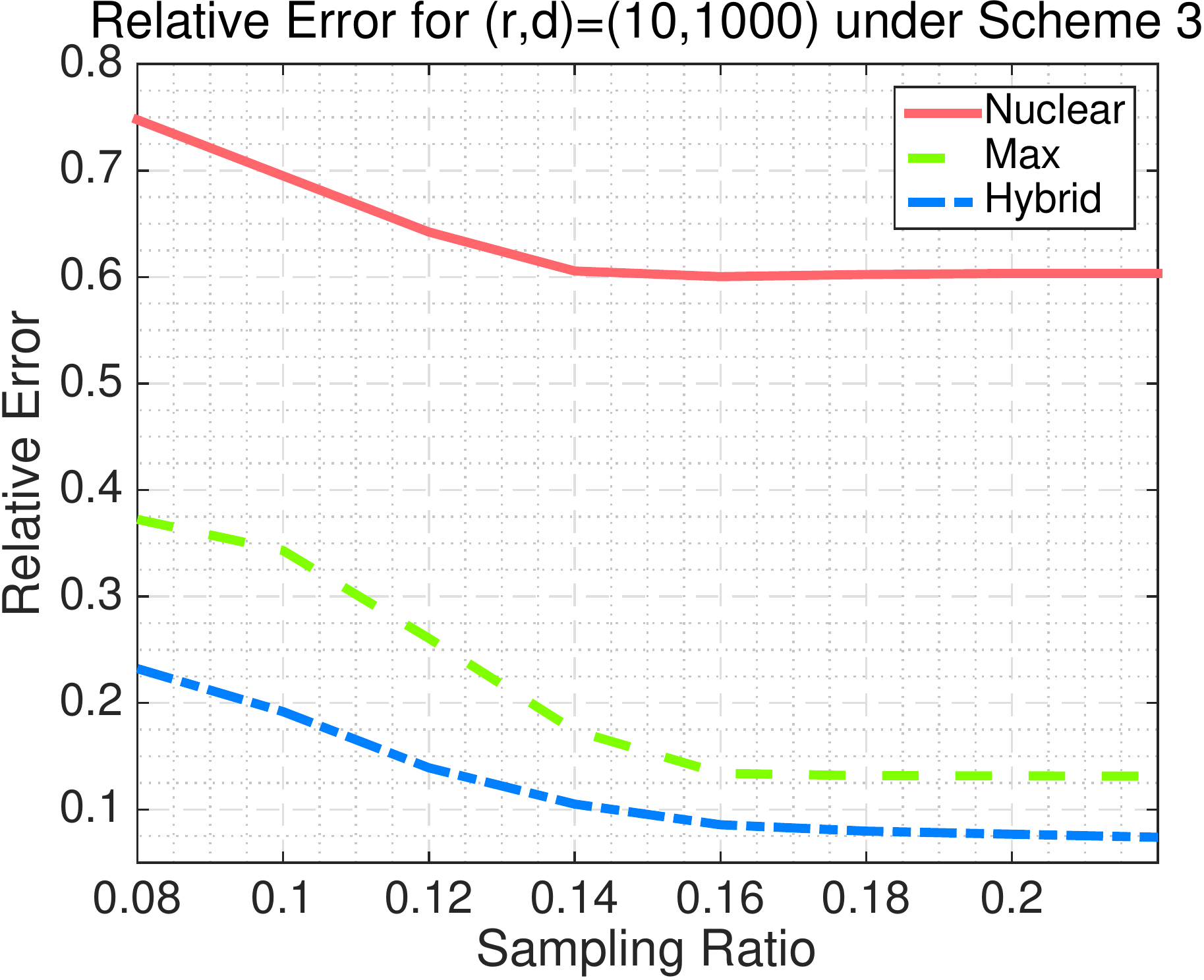}
}
\caption{Relative Errors under different settings for the noisy case under Scheme 3 (non-uniform sampling). }
\label{fig:s3}
\end{figure}
 
\subsection{Real Datasets}
In this subsection, we test our methods using some real datasets. We first consider the well-known Jester  joke dataset. This dataset contains more than 4.1 million ratings for 100 jokes from 73,421 users, and it is publicly available through \url{http://www.ieor.berkeley.edu/~goldberg/jester-data/}.  
The whole Jester joke dataset contains three sub-datasets, which are: (1) jester-1: 24,983 users who rate 36 or more jokes; (2) jester-2: 23,500 users who rate 36 or more jokes; (3) jester-3: 24,938 users who rate between 15 and 35 jokes. More  detailed descriptions can be found in \cite{toh2010accelerated} and \cite{chen2012matrix}, where the nuclear-norm based approach is used to study this dataset.

 Due to the large number of users, as in \cite{chen2012matrix}, we randomly select $n_u$ users' ratings from the datasets. Since many entries are unknown, we cannot compute the relative error as we did for the simulated data. Instead, we take the metric of the normalized mean absolute error (NMAE) to measure the accuracy of the estimator $\hat{M}$:
 $$
 \text{NMAE} = \frac{\sum_{(j,k) \not\in\Omega} |\hat{M}_{jk} - M^0_{jk}|  }{|\Omega| (r_{\max}-r_{\min})},
 $$
 where $r_{\min}$ and $r_{\max}$ denote the lower and upper bounds for the ratings, respectively. In the Jester joke dataset, the range is $[-10,10]$. Thus, we have  $r_{\max}-r_{\min} = 20$. 
 
 In each iteration, we first randomly select $n_u$ users, and then randomly permute the ratings from the users to generate $M^0\in\RR^{n_u\times 100}$. Next, we adopt the generating scheme used in Scheme 2 in the previous subsection to generate a set of  observed indices $\Omega$. Note that we can only observe the entry $(j,k)$ if $(j,k)\in\Omega$, and $M^0_{j,k}$ is available. Thus, the actual sampling ratio is less than the input SR. We consider different settings of $n_u$ and SR, and we report the averaged NMAE and running times in Table \ref{tab:jester} after running each setting five times.
It can be seen that the max-norm and hybrid approaches outperform the nuclear-norm approach in all cases. %Though we do not get a very large improvement over the nuclear-norm approach, but it is well-known that nuclear-norm is already a very good approach in practice, and getting some advantages are very promising that could lead to s. 
 This provides strong evidences that the proposed estimator and algorithm could be useful in practice.
 
Meanwhile, we observe that the running times for solving  {the}  max-norm
penalized optimization problems are significantly longer than that for solving {the} nuclear-norm penalized problem. 
This is because solving max-norm penalized optimization problems is intrinsically more difficult than solving nuclear-norm penalized ones. Specifically, in Algorithm 1, the most computationally expensive step is to {compute}
 a full eigenvalue decomposition of a matrix of size $d_1+d_2$ by $d_1+d_2$ during the $X$-update step. As a comparison, in nuclear-norm regularized optimizations, we only need to 
%conduct 
{compute}
a singular value decomposition of a matrix of size $d_1$ by $d_2$. In the Jester joke dataset, since $d_2\ll d_1$, singular value decomposition takes the advantage of a small $d_2$, but the computational cost of the max-norm approach is dominated by the large $d_1$. In practical matrix completion problems, the computational efficiency is 
%usually not given the top priority, 
%and we pay more attentions on reducing the reconstruction error.
{sometimes not the top priority, but more attention is placed on
reducing the reconstruction error.}
 Thus, depending on the specific application, the max-norm and hybrid approaches provide useful complements to the nuclear-norm approach.
 
We also consider the MovieLens data. 
%\citep{herlocker1999algorithmic}. 
The dataset is available through \url{http://www.grouplens.org}. We first implement the proposed methods on the Movie-100K dataset, which contains 100,000 ratings for 1,682 movies by 943 users. The ratings range from $r_{\min} = 1$ to $r_{\max} = 5$. In this experiment, we first randomly permute the rows and columns of the matrix, and then sample the observed entries as in Scheme 2 in the previous subsection. Table \ref{tab:movie} reports the averaged NMAE and running times of different methods.
{
Next, we {implement} the methods on the Movie-1M dataset. This dataset contains 1,000,209 ratings of 3,900 movies made by 6,040 users. We randomly select $n$ users and $n$ movies to conduct the tests, where $n = 1500$ or 2000. We report the results in Table \ref{tab:movie2}. 
} {From} Tables \ref{tab:movie} and \ref{tab:movie2},  we observe that the max-norm and hybrid approaches lead to better matrix recovery results than the nuclear-norm approach in all cases. In addition, we observe that the differences between running times of the max-norm and nuclear-norm approaches are less significant than those in the Jester joke problem. This is because $d_1$ and $d_2$ are of the same order in the MovieLens example. %, and thecomputational costs of an eigenvalue decomposition of a $d_1+d_2$ by $d_1+d_2$ matrix and a singular value decomposition of a $d_1$ by $d_2$ matrix become similar. 
Therefore, in practice, if the computational efficiency %has
is the top
 priority, and if $d_1\ll d_2$ or $d_1\gg d_2$, the nuclear-norm approach is preferable. While if controlling the reconstruction accuracy  {attracts more concern,}
 % is more 
%concerned, {important,}
we recommend the proposed hybrid approach.

{
\begin{remark}
Note that the improvement from the hybrid and max-norm approaches over the nuclear-norm approach is about 5\%, which looks marginal. However, a 5\% improvement can be significant  in practice as the nuclear-norm approach is widely recognized as a highly efficient approach. In the earlier Netflix competition, it is seen that the results from top teams (where nuclear-norm approach is used as part of the algorithms) are all very close, and a 5\% improvement can be significant for practitioners. See \url{http://www.research.att.com/articles/featured_stories/2010_05/201005_netflix2_article.html?fbid=pgKJkRJ5mbi}. In addition,  though the nuclear-norm approach is computationally more efficient, we note that in this particular application, computation efficiency is not of the highest priority, and the modest sacrifice of computational cost is tolerable here.
\end{remark}
}

 \begin{table}[ht]
\centering
\caption{Averaged relative error and running time in seconds for different methods under uniform sampling scheme. Under noiseless and noisy settings, for the nuclear norm approach, we set $\mu = 1\times10^{-4}\|Y_\Omega\|_F$ and $2\times 10^{-4}\|Y_\Omega\|_F$. For the max-norm approach, we set $\lambda=2\norm{Y_\Omega}_F$ and $0.05\norm{Y_\Omega}_F$. For the hybrid approach, and we set  $\lambda=0.01\norm{Y_\Omega}_F$ and $0.8\norm{Y_\Omega}_F$, $\mu=0.02\lambda$ and $1\times 10^{-4}\lambda$.}
\medskip
\begin{tabular}{r*{9}{r}}
\toprule
\toprule
& & &\multicolumn{2}{c}{Nuclear}  &\multicolumn{2}{c}{Max} &\multicolumn{2}{c}{Hybrid}
\\
%&\multicolumn{2}{c}{($\lambda = 2\times10^{-4}\|Y_S\|_F$)} &\multicolumn{2}{c}{$(\lambda=0.1\norm{Y_S}_F,\mu=0)$} &\multicolumn{2}{c}{$(\lambda=0.2\norm{Y_S}_F,\mu=2\times10^{-4}\lambda)$}  \\
\cmidrule{4-5}  \cmidrule{6-7}  \cmidrule{8-9}
\multicolumn{1}{r} {$\sigma$} & \multicolumn{1}{r} {$d_t $} & \multicolumn{1}{r} {$(r,\text{SR}) $}& \multicolumn{1}{c}{RE} &  \multicolumn{1}{r}{Time} & \multicolumn{1}{c}  {RE} & \multicolumn{1}{r} {Time}  & \multicolumn{1}{c}  {RE} & \multicolumn{1}{r} {Time}\\
\midrule
0 & 500 & $(5, 0.10)$  & $1.1\times 10^{-3}$ & 6.9      & $4.1\times 10^{-2}$ & 12.0     & $4.0\times 10^{-2}$  & 12.5\\
\ & \ & $(5, 0.15)$ & $7.7\times 10^{-4}$  & 7.0     & $3.8\times 10^{-2}$ & 11.1      & $2.9\times 10^{-2}$  &  13.4\\
%\ & \ &  $(5, 0.20)$ & $6.1\times 10^{-4}$ & 39.28     & $2.7\times 10^{-2}$ &  77.10     & $2.7\times 10^{-2}$  & 65.07\\
\ & \ & $(10, 0.10)$ & $5.5\times 10^{-2}$ & 8.0     & {$1.2\times 10^{-1}$} & 11.4      & $2.9\times 10^{-2}$  & 13.4\\
\ & \ & $(10, 0.15)$ & $1.3\times 10^{-3}$ & 8.6     & $3.8\times 10^{-2}$ & 11.7      & $2.3\times 10^{-2}$  & 12.0\\
%\ & \ & $(10, 0.20)$ & $1.0\times 10^{-3}$ & 8.7     & $2.8\times 10^{-2}$ & 10.4      & $2.1\times 10^{-2}$  & 12.4\\
\midrule
\ & 1000  & $(5, 0.10)$  & $8.8\times 10^{-4}$ & 44.8      & $2.9\times 10^{-2}$ & 110.4     & $1.9\times 10^{-2}$  & 115.1\\
\ & \ & $(5, 0.15)$ & $6.6\times 10^{-4}$  & 43.3     & $1.9\times 10^{-2}$ & 111.3      & $1.8\times 10^{-2}$  &  114.3\\
\ & \ &  $(5, 0.20)$ & $5.5\times 10^{-4}$ & 44.6     & $1.8\times 10^{-2}$ &  112.4     & $6.7\times 10^{-3}$  & 120.0\\
\ & \ & $(10, 0.10)$ & $1.5\times 10^{-3}$ & 44.4     & $2.9\times 10^{-2}$ & 108.7      & $2.0\times 10^{-2}$  & 121.7\\
\ & \ & $(10, 0.15)$ & $1.0\times 10^{-3}$ & 45.8     & $2.0\times 10^{-2}$ & 112.8      & $1.3\times 10^{-2}$  & 117.8\\
\ & \ & $(10, 0.20)$ & $8.3\times 10^{-4}$ & 45.5     & $1.5\times 10^{-2}$ & 110.8      & $8.9\times 10^{-3}$  & 117.3\\
\midrule
\ & 1500  & $(5, 0.10)$  & $8.1\times 10^{-4}$ & 162.8      & $2.3\times 10^{-2}$ & 385.4   & $1.2\times 10^{-2}$  & 408.2\\
\ & \ & $(5, 0.15)$ & $6.3\times 10^{-4}$  & 158.3     & $1.7\times 10^{-2}$ & 396.9      & $1.1\times 10^{-2}$  &  406.6\\
\ & \ &  $(5, 0.20)$ & $5.3\times 10^{-4}$ & 158.1     & $1.3\times 10^{-2}$ &  410.9    & $5.6\times 10^{-3}$  &   405.3 \\
\ & \ & $(10, 0.10)$ & $1.3\times 10^{-3}$ & 165.9     & $2.0\times 10^{-2}$ & 413.8      & $1.5\times 10^{-2}$  &  413.3\\
\ & \ & $(10, 0.15)$ & $9.5\times 10^{-4}$ & 160.8     & $1.4\times 10^{-2}$ & 410.1      & $1.3\times 10^{-2}$  & 423.2\\
\ & \ & $(10, 0.20)$ & $7.8\times 10^{-4}$ & 161.0     & $1.2\times 10^{-2}$ & 395.1      & $7.0\times 10^{-3}$  & 398.2\\
\midrule
\midrule
$0.01$ & 500 & $(5, 0.10)$  & $7.4\times 10^{-2}$ & 6.4      & $6.4\times 10^{-2}$ & 10.5     & $6.3\times 10^{-2}$  & 12.3\\
\ & \ & $(5, 0.15)$ & $5.4\times 10^{-3}$  & 6.4     & $4.8\times 10^{-2}$ & 11.4      & $4.3\times 10^{-2}$  &  13.1\\
%\ & \ &  $(5, 0.20)$ & $6.1\times 10^{-4}$ & 39.28     & $2.7\times 10^{-2}$ &  77.10     & $2.7\times 10^{-2}$  & 65.07\\
\ & \ & $(10, 0.10)$ & {$1.7\times 10^{-1}$} & 6.3     & $5.2\times 10^{-2}$ & 10.9      & $6.6\times 10^{-2}$  & 11.9\\
\ & \ & $(10, 0.15)$ & {$7.8\times 10^{-2}$} & 6.5     & $4.0\times 10^{-2}$ & 11.2      & $4.8\times 10^{-2}$  & 14.2\\
%\ & \ & $(10, 0.20)$ & $5.9\times 10^{-2}$ & 6.6     & $3.9\times 10^{-2}$ & 11.8      & $4.1\times 10^{-2}$  & 13.5\\
\midrule
\ & 1000  & $(5, 0.10)$  & $4.8\times 10^{-2}$ & 47.1      & $3.9\times 10^{-2}$ & 101.7     & $3.6\times 10^{-2}$  & 119.8\\
\ & \ & $(5, 0.15)$ & $4.5\times 10^{-2}$  & 47.5     & $2.8\times 10^{-2}$ & 106.8      & $3.3\times 10^{-2}$  &  116.6\\
\ & \ &  $(5, 0.20)$ & $4.7\times 10^{-2}$ & 47.6     & $2.6\times 10^{-2}$ &  117.3     & $2.6\times 10^{-2}$  &   119.8\\
\ & \ & $(10, 0.10)$ & $6.2\times 10^{-2}$ & 47.1     & $4.3\times 10^{-2}$ & 106.1      & $4.2\times 10^{-2}$  &  116.7\\
\ & \ & $(10, 0.15)$ & $4.9\times 10^{-2}$ & 47.2     & $3.3\times 10^{-2}$ & 105.9      & $3.0\times 10^{-2}$  &  120.2\\
\ & \ & $(10, 0.20)$ & $4.5\times 10^{-2}$ & 47.7     & $2.7\times 10^{-2}$ & 112.2      & $3.2\times 10^{-3}$  &  120.3\\
\midrule
\ & 1500  & $(5, 0.10)$  & $4.2\times 10^{-2}$ & 161.2      & $2.9\times 10^{-2}$ & 377.9  & $2.9 \times 10^{-2}$  & 406.1\\
\ & \ & $(5, 0.15)$ & $4.1\times 10^{-2}$  &  167.5     & $2.4\times 10^{-2}$ &  408.7     & $2.8\times 10^{-2}$  &  409.3\\
\ & \ &  $(5, 0.20)$ & $4.4\times 10^{-2}$ &   153.4     & $2.1\times 10^{-2}$ &  412.9    & $2.1\times 10^{-2}$  &   415.6 \\
\ & \ & $(10, 0.10)$ & $5.0\times 10^{-3}$ &  166.9    & $3.3\times 10^{-2}$ &  397.2      & $3.3\times 10^{-2}$  &  404.6\\
\ & \ & $(10, 0.15)$ & $4.7\times 10^{-3}$ &  160.8    & $2.6\times 10^{-2}$ &  395.4      & $2.5\times 10^{-2}$  &  424.2\\
\ & \ & $(10, 0.20)$ & $4.3\times 10^{-3}$ &  150.6     & $2.1\times 10^{-2}$ &  401.9      & $2.0\times 10^{-2}$  & 380.7\\
  \bottomrule\bottomrule 
\end{tabular}
\label{tab:comadm1}
\end{table}

\begin{table}[ht]
\centering
\caption{Averaged relative error and running time in seconds for different methods using non-uniformly sampled data as in Scheme 2. For the nuclear norm approach, we set $\mu = 2\times10^{-4}\|Y_\Omega\|_F$. For the max-norm approach, we set $\lambda=0.1\norm{Y_\Omega}_F$. For the hybrid approach, and we set  $\lambda=0.2\norm{Y_\Omega}_F,\mu=2\times10^{-4}\lambda$.}
\medskip
\begin{tabular}{r*{8}{r}}
\toprule
\toprule
& & &\multicolumn{2}{c}{Nuclear}  &\multicolumn{2}{c}{Max} &\multicolumn{2}{c}{Hybrid}
\\
%&\multicolumn{2}{c}{($\lambda = 2\times10^{-4}\|Y_S\|_F$)} &\multicolumn{2}{c}{$(\lambda=0.1\norm{Y_S}_F,\mu=0)$} &\multicolumn{2}{c}{$(\lambda=0.2\norm{Y_S}_F,\mu=2\times10^{-4}\lambda)$}  \\
\cmidrule{4-5}  \cmidrule{6-7}  \cmidrule{8-9}
\multicolumn{1}{r} {$\sigma$} & \multicolumn{1}{r} {$d_t $} & \multicolumn{1}{r} {$(r,\text{SR}) $}& \multicolumn{1}{c}{RE} &  \multicolumn{1}{r}{Time} & \multicolumn{1}{c}  {RE} & \multicolumn{1}{r} {Time}  & \multicolumn{1}{c}  {RE} & \multicolumn{1}{r} {Time}\\
\midrule
0 & 500 & $(5, 0.10)$  & $7.4\times 10^{-1}$ & 7.6      & $2.2\times 10^{-1}$ & 12.5     & $1.2\times 10^{-1}$  & 15.8\\
\ & \ & $(5, 0.15)$ & $6.1\times 10^{-1}$  & 7.8     & $9.6\times 10^{-2}$ & 13.1      & $6.1\times 10^{-2}$  &  15.7\\
%\ & \ &  $(5, 0.20)$ & $6.1\times 10^{-4}$ & 39.28     & $2.7\times 10^{-2}$ &  77.10     & $2.7\times 10^{-2}$  & 65.07\\
\ & \ & $(10, 0.10)$ & $7.7\times 10^{-1}$ & 7.5     & $2.1\times 10^{-1}$ & 12.9      & $1.6\times 10^{-1}$  & 16.1\\
\ & \ & $(10, 0.15)$ & $6.1\times 10^{-1}$ & 8.5     & $6.1\times 10^{-2}$ & 13.0      & $7.6\times 10^{-2}$  & 15.7\\
%\ & \ & $(10, 0.20)$ & $6.1\times 10^{-1}$ & 8.4     & $8.5\times 10^{-2}$ & 12.7      & $7.5\times 10^{-2}$  & 16.4\\
\midrule
\ & 1000  & $(5, 0.10)$  & $7.4\times 10^{-1}$ & 45.2      & $2.2\times 10^{-1}$ & 97.0     & $1.1\times 10^{-1}$  & 113.9\\
\ & \ & $(5, 0.15)$ & $6.1\times 10^{-1}$  & 48.2     & $1.2\times 10^{-1}$ & 104.0     & $4.3\times 10^{-2}$  &  113.1\\
\ & \ &  $(5, 0.20)$ & $6.2\times 10^{-1}$ & 45.4     & $1.1\times 10^{-1}$ &  105.6     & $3.5\times 10^{-2}$  & 105.0\\
\ & \ & $(10, 0.10)$ & $7.5\times 10^{-1}$ & 45.8     & $1.9\times 10^{-1}$ & 97.3      & $8.8\times 10^{-2}$  & 113.8\\
\ & \ & $(10, 0.15)$ & $6.0\times 10^{-1}$ & 47.6     & $5.9\times 10^{-2}$ & 105.2      & $4.1\times 10^{-2}$  & 109.7\\
\ & \ & $(10, 0.20)$ & $6.0\times 10^{-1}$ & 44.6     & $6.1\times 10^{-2}$ & 108.8      & $4.3\times 10^{-2}$  & 108.2\\
\midrule
\ & 1500  & $(5, 0.10)$  & $7.5\times 10^{-1}$ & 143.2     & $2.3\times 10^{-1}$ & 388.7     & $1.0\times 10^{-1}$  & 372.3\\
\ & \ & $(5, 0.15)$ & $6.0\times 10^{-1}$  & 147.2     & $1.3\times 10^{-1}$ & 398.0      & $6.2\times 10^{-2}$  &  389.0\\
\ & \ &  $(5, 0.20)$ & $6.0\times 10^{-1}$ & 138.5     & $1.1\times 10^{-1}$ &  397.6    & $2.2\times 10^{-2}$  & 358.8\\
\ & \ & $(10, 0.10)$ & $7.5\times 10^{-1}$ & 143.2     & $1.4\times 10^{-1}$ & 360.0    & $7.4\times 10^{-2}$  & 386.1\\
\ & \ & $(10, 0.15)$ & $6.0\times 10^{-1}$ & 142.3     & $5.9\times 10^{-2}$ & 392.3      & $2.8\times 10^{-2}$  & 380.2\\
\ & \ & $(10, 0.20)$ & $6.0\times 10^{-1}$ & 137.1     & $9.9\times 10^{-2}$ & 395.2      & $2.4\times 10^{-2}$  & 359.4\\
\midrule
\midrule
0.01 & 500 & $(5, 0.10)$  & $7.4\times 10^{-1}$ & 7.5     & $2.2\times 10^{-1}$ & 15.1     & $1.3\times 10^{-1}$  & 16.2\\
\ & \ & $(5, 0.15)$ & $6.1\times 10^{-1}$  & 8.3     & $1.0\times 10^{-1}$ & 14.9      & {$7.1\times 10^{-2}$}  &  16.2\\
%\ & \ &  $(5, 0.20)$ & $6.1\times 10^{-4}$ & 39.28     & $2.7\times 10^{-2}$ &  77.10     & $2.7\times 10^{-2}$  & 65.07\\
\ & \ & $(10, 0.10)$ & $7.7\times 10^{-1}$ & 8.7     & $2.4\times 10^{-1}$ & 15.5      & $1.7\times 10^{-1}$  & 16.2\\
\ & \ & $(10, 0.15)$ & $6.2\times 10^{-1}$ & 8.3     & $8.0\times 10^{-2}$ & 15.2      & $8.6\times 10^{-2}$  & 16.5\\
%\ & \ & $(10, 0.20)$ & $6.2\times 10^{-1}$ & 8.4     & $7.0\times 10^{-2}$ & 15.4      & $8.3\times 10^{-2}$  & 16.1\\
\midrule
\ & 1000  & $(5, 0.10)$  & $7.4\times 10^{-1}$ & 44.5   & $2.2\times 10^{-1}$ & 117.9     & $1.0\times 10^{-1}$  & 118.2\\
\ & \ & $(5, 0.15)$ & $6.1\times 10^{-1}$  & 47.0     & $1.2\times 10^{-1}$ &  116.9    & $5.2\times 10^{-2}$  &   120.8\\
\ & \ &  $(5, 0.20)$ & $6.2\times 10^{-1}$ & 46.7     & $1.1\times 10^{-1}$ &  120.7     & $4.3\times 10^{-2}$  &  123.0\\
\ & \ & $(10, 0.10)$ & $7.5\times 10^{-1}$ & 45.6     & $2.0\times 10^{-1}$ & 117.3      & $9.3\times 10^{-2}$  & 122.9\\
\ & \ & $(10, 0.15)$ & $6.1\times 10^{-1}$ & 47.3     & $6.5\times 10^{-2}$ & 119.3      & $5.3\times 10^{-2}$  & 123.3\\
\ & \ & $(10, 0.20)$ & $6.1\times 10^{-1}$ & 46.3     & $6.3\times 10^{-2}$ & 123.2      & $5.0\times 10^{-2}$  & 120.5\\
\midrule
\ & 1500  & $(5, 0.10)$  & $7.5\times 10^{-1}$ & 152.6     & $2.3\times 10^{-1}$ & 395.6   & $7.2\times 10^{-2}$  & 396.9\\
\ & \ & $(5, 0.15)$ & $6.0\times 10^{-1}$  & 156.3     & $1.2\times 10^{-1}$ &   382.0     & $5.3\times 10^{-2}$  &  394.2\\
\ & \ &  $(5, 0.20)$ & $6.0\times 10^{-1}$ & 162.4     & $1.1\times 10^{-1}$ &  396.3    & $3.0\times 10^{-2}$  & 398.2\\
\ & \ & $(10, 0.10)$ & $7.5\times 10^{-1}$ & 154.5     & $1.4\times 10^{-1}$ & 403.2    & $7.3\times 10^{-2}$  & 406.1\\
\ & \ & $(10, 0.15)$ & $6.0\times 10^{-1}$ & 158.7     & $5.9\times 10^{-2}$ & 396.5      & $4.3\times 10^{-2}$  & 399.1\\
\ & \ & $(10, 0.20)$ & $6.0\times 10^{-1}$ & 157.7     & $9.5\times 10^{-2}$ & 405.4      & $3.6\times 10^{-2}$  & 400.3\\
  \bottomrule\bottomrule 
\end{tabular}

\label{tab:comadm2}
\end{table}

\begin{table}[ht]
\centering
\caption{Averaged relative error and running time in seconds for the nuclear-norm and max-norm penalized matrix completion using non-uniformly sampled data as in Scheme 3. 
%For the nuclear norm approach, we set $\mu = 2\times10^{-4}\|Y_S\|_F$. 
%For the max-norm approach, we set $\lambda=0.1\norm{Y_S}_F$. For the hybrid approach, and we set  $\lambda=0.2\norm{Y_S}_F,\mu=2\times10^{-4}\lambda$.
The parameters are chosen to be the same as those in Table 
\ref{tab:comadm2}.
}
\medskip
\begin{tabular}{r*{8}{r}}
%\vspace{0.5pt}
\toprule
\toprule
& & &\multicolumn{2}{c}{Nuclear}  &\multicolumn{2}{c}{Max} &\multicolumn{2}{c}{Hybrid}
\\
%&\multicolumn{2}{c}{($\lambda = 2\times10^{-4}\|Y_S\|_F$)} &\multicolumn{2}{c}{$(\lambda=0.1\norm{Y_S}_F,\mu=0)$} &\multicolumn{2}{c}{$(\lambda=0.2\norm{Y_S}_F,\mu=2\times10^{-4}\lambda)$}  \\
\cmidrule{4-5}  \cmidrule{6-7}  \cmidrule{8-9}
\multicolumn{1}{r} {$\sigma$} & \multicolumn{1}{r} {$d_t $} & \multicolumn{1}{r} {$(r,\text{SR}) $}& \multicolumn{1}{c}{RE} &  \multicolumn{1}{r}{Time} & \multicolumn{1}{c}  {RE} & \multicolumn{1}{r} {Time}  & \multicolumn{1}{c}  {RE} & \multicolumn{1}{r} {Time}\\
\midrule
0 & 500 & $(5, 0.10)$  & $7.4\times 10^{-1}$ & 7.2      & $2.6\times 10^{-1}$ & 14.7     & $1.9\times 10^{-1}$  & 17.8\\
\ & \ & $(5, 0.15)$ & $7.2\times 10^{-1}$  & 7.3     & $1.9\times 10^{-1}$ &  14.8     & $8.6\times 10^{-2}$  &  16.7\\
%\ & \ &  $(5, 0.20)$ & $6.1\times 10^{-4}$ & 39.28     & $2.7\times 10^{-2}$ &  77.10     & $2.7\times 10^{-2}$  & 65.07\\
\ & \ & $(10, 0.10)$ & $8.0\times 10^{-1}$ & 7.3     & $3.9\times 10^{-1}$ & 13.9      & $3.2\times 10^{-1}$  & 17.6\\
\ & \ & $(10, 0.15)$ & $7.1\times 10^{-1}$ & 7.4     & $1.5\times 10^{-1}$ & 14.6      & $1.1\times 10^{-1}$  & 17.9\\
%\ & \ & $(10, 0.20)$ & $6.1\times 10^{-1}$ & 7.7     & $8.3\times 10^{-2}$ & 14.5      & $7.3\times 10^{-2}$  & 16.9\\
\midrule
\ & 1000  & $(5, 0.10)$  & $7.4\times 10^{-1}$ & 42.4      & $2.4\times 10^{-1}$ & 120.6    & $1.5\times 10^{-1}$  & 121.6\\
\ & \ & $(5, 0.15)$ & $7.1\times 10^{-1}$  & 42.1    & $1.9\times 10^{-1}$ &   115.7     & $7.9\times 10^{-2}$  &  119.9\\
\ & \ &  $(5, 0.20)$ & $6.2\times 10^{-1}$ & 44.2     & $1.2\times 10^{-1}$ &  118.2     & $3.9\times 10^{-2}$  & 119.8\\
\ & \ & $(10, 0.10)$ & $7.5\times 10^{-1}$ & 42.9     & $1.9\times 10^{-1}$ &  110.5      & $1.3\times 10^{-1}$  & 119.9\\
\ & \ & $(10, 0.15)$ & $7.1\times 10^{-1}$ & 42.8     & $1.4\times 10^{-1}$ &  115.7     & $6.6\times 10^{-2}$  & 119.0\\
\ & \ & $(10, 0.20)$ & $6.0\times 10^{-1}$ & 44.1     & $7.0\times 10^{-2}$ &  118.7     & $3.7\times 10^{-2}$  & 119.6\\
\midrule
\ & 1500  & $(5, 0.10)$  & $7.5\times 10^{-1}$ & 142.1      & $2.4\times 10^{-1}$ & 391.7     & $1.6\times 10^{-1}$  & 380.7\\
\ & \ & $(5, 0.15)$ & $7.1\times 10^{-1}$  & 143.8     & $2.1\times 10^{-1}$ &  385.4      & $7.5\times 10^{-2}$  &  386.4\\
\ & \ &  $(5, 0.20)$ & $6.0\times 10^{-1}$ & 146.6     & $1.1\times 10^{-1}$ &  385.0   & $2.9\times 10^{-2}$  & 387.9\\
\ & \ & $(10, 0.10)$ & $7.5\times 10^{-1}$ & 143.1     & $1.7\times 10^{-1}$ & 372.9    & $1.1\times 10^{-1}$  & 377.9\\
\ & \ & $(10, 0.15)$ & $7.1\times 10^{-1}$ & 144.2     & $1.6\times 10^{-2}$ & 390.4    & $3.9\times 10^{-2}$  & 388.5\\
%\ & \ & $(10, 0.20)$ & $6.0\times 10^{-1}$ & 146.5     & $8.5\times 10^{-2}$ & 391.4      & $2.6\times 10^{-2}$  & 388.7\\
\midrule
\midrule
0.01 & 500 & $(5, 0.10)$  & $7.5\times 10^{-1}$ & 7.5     & $4.1\times 10^{-2}$ & 13.7    & $4.0\times 10^{-2}$  & 15.4\\
\ &\ & $(5, 0.15)$ & $7.2\times 10^{-1}$  & 7.8     & $3.8\times 10^{-2}$ & 13.7      & $2.9\times 10^{-2}$  &  15.1\\
%\ & \ &  $(5, 0.20)$ & $6.1\times 10^{-4}$ & 39.28     & $2.7\times 10^{-2}$ &  77.10     & $2.7\times 10^{-2}$  & 65.07\\
\ & \ & $(10, 0.10)$ & $8.0\times 10^{-1}$ & 7.5     & $1.2\times 10^{-1}$ & 12.9      & $2.9\times 10^{-2}$  & 16.1\\
\ & \ & $(10, 0.15)$ & $7.1\times 10^{-1}$ & 7.8     & $3.8\times 10^{-2}$ & 13.8      & $2.3\times 10^{-2}$  & 16.3\\
\ & \ & $(10, 0.20)$ & $6.2\times 10^{-1}$ & 8.5     & $2.8\times 10^{-2}$ & 13.8      & $2.1\times 10^{-2}$  & 16.2\\
\midrule
\ & 1000  & $(5, 0.10)$  & $7.4\times 10^{-1}$ & 44.4      & $2.4\times 10^{-1}$ & 115.9     & $1.5\times 10^{-1}$  & 118.3\\
\ & \ & $(5, 0.15)$ & $7.1\times 10^{-1}$  & 45.6     & $1.9\times 10^{-1}$ & 117.6      & $7.7\times 10^{-2}$  &  119.1\\
\ & \ &  $(5, 0.20)$ & $6.2\times 10^{-1}$ & 47.8     & $1.1\times 10^{-1}$ &  117.1     & $4.4\times 10^{-2}$  &120.0\\
\ & \ & $(10, 0.10)$ & $7.5\times 10^{-1}$ & 44.6     & $2.0\times 10^{-1}$ & 112.3      & $1.4\times 10^{-1}$  & 118.0\\
\ & \ & $(10, 0.15)$ & $7.1\times 10^{-1}$ & 45.6     & $1.4\times 10^{-1}$ & 117.3      & $6.6\times 10^{-2}$  & 117.6\\
\ & \ & $(10, 0.20)$ & $6.1\times 10^{-1}$ & 48.3     & $7.0\times 10^{-2}$ & 113.4      & $4.7\times 10^{-2}$  & 119.4\\
\midrule
\ & 1500  & $(5, 0.10)$  & $7.5\times 10^{-1}$ & 148.2      & $2.4\times 10^{-1}$ & 381.7     & $1.6\times 10^{-1}$  & 386.9\\
\ & \ & $(5, 0.15)$ & $7.1\times 10^{-1}$  & 150.4     & $2.1\times 10^{-1}$ & 396.8     & $6.5\times 10^{-2}$  &  396.1\\
\ & \ &  $(5, 0.20)$ & $6.0\times 10^{-1}$ & 156.2     & $1.1\times 10^{-1}$ &  396.9     & $3.2\times 10^{-2}$  & 390.0\\
\ & \ & $(10, 0.10)$ & $7.5\times 10^{-1}$ & 148.6     & $1.7\times 10^{-1}$ & 401.5     & $1.1\times 10^{-1}$  & 396.9\\
\ & \ & $(10, 0.15)$ & $7.1\times 10^{-1}$ & 151.4     & $1.6\times 10^{-1}$ & 405.3      & $4.8\times 10^{-2}$  & 389.2\\
\ & \ & $(10, 0.20)$ & $6.0\times 10^{-1}$ & 160.1     & $8.0\times 10^{-2}$ & 398.4     & $3.7\times 10^{-2}$  & 393.1\\
  \bottomrule\bottomrule 
\end{tabular}
\label{tab:comadm3}
\end{table}

 \begin{table}[ht]
\centering
\caption{Averaged normalized mean absolute error  and running time in seconds 
for different methods using Jester joke dataset. For the nuclear norm approach, we set $\mu = 2\times10^{-4}\|Y_S\|_F$. For the max-norm approach, we set $\lambda=0.5\norm{Y_S}_F$. For the hybrid approach, we set  $\lambda=0.8\norm{Y_S}_F,\mu=1\times10^{-4}\lambda$.}
\medskip
\begin{tabular}{r*{8}{r}}
\toprule
\toprule
\ &  \ &\multicolumn{2}{c}{Nuclear}  &\multicolumn{2}{c}{Max} &\multicolumn{2}{c}{Hybrid}
\\
%&\multicolumn{2}{c}{($\lambda = 2\times10^{-4}\|Y_S\|_F$)} &\multicolumn{2}{c}{$(\lambda=0.1\norm{Y_S}_F,\mu=0)$} &\multicolumn{2}{c}{$(\lambda=0.2\norm{Y_S}_F,\mu=2\times10^{-4}\lambda)$}  \\
\cmidrule{3-4}  \cmidrule{5-6}  \cmidrule{7-8}
\multicolumn{1}{c} {Example}  & \multicolumn{1}{r} {$(n_u,\text{SR})$}& \multicolumn{1}{c}{NMAE} &  \multicolumn{1}{r}{Time} & \multicolumn{1}{c}  {NMAE} & \multicolumn{1}{r} {Time}  & \multicolumn{1}{c}  {NMAE} & \multicolumn{1}{r} {Time}\\
\midrule
jester-1& (1000, 0.15) & 0.210 & 4.82 & 0.197 & 110.55 & 0.200 & 87.13\\
\ & (1000, 0.20) & 0.209 & 4.83 & 0.194 & 111.79 & 0.203 & 89.98\\
\ & (1000, 0.25) & 0.204 & 5.12 & 0.188 & 111.36 & 0.197 & 89.02\\
\ & (1500, 0.15) & 0.210 & 5.93 & 0.194 & 302.47 & 0.201 & 250.07\\
\ & (1500, 0.20) & 0.206 & 6.08 & 0.192 & 307.70 & 0.195 & 255.29 \\
\ & (1500, 0.25) & 0.204 & 6.39 & 0.185 & 305.91 & 0.194 & 254.66 \\
\ & (2000, 0.15) & 0.212 & 7.06 & 0.192 & 647.25 & 0.196 & 566.84\\
\ & (2000, 0.20) & 0.208 & 7.30 & 0.188 & 671.73 & 0.192 & 547.89 \\
\ & (2000, 0.25) & 0.205 & 7.45 & 0.183 & 640.75 & 0.192 & 558.02 \\
\midrule
jester-2& (1000, 0.15) & 0.211 & 4.86 & 0.199 & 109.15 & 0.196 & 86.34\\
\ & (1000, 0.20) & 0.207 & 5.01 & 0.192 & 110.40 & 0.193 & 87.81\\
\ & (1000, 0.25) & 0.204 & 4.89 & 0.188 & 110.41 & 0.187 & 90.07\\
\ & (1500, 0.15) & 0.212 & 5.86 & 0.197 & 313.01 & 0.198 & 247.26\\
\ & (1500, 0.20) & 0.210 & 6.10 & 0.192 & 313.39 & 0.193 & 260.84 \\
\ & (1500, 0.25) & 0.205 & 6.34 & 0.189 & 322.05 & 0.187 & 255.88 \\
\ & (2000, 0.15) & 0.213 & 6.99 & 0.197 & 633.97 & 0.198 & 577.32\\
\ & (2000, 0.20) & 0.208 & 7.50 & 0.194 & 644.04 & 0.193 & 562.32 \\
\ & (2000, 0.25) & 0.204 & 7.42 & 0.187 & 687.24 & 0.188 & 576.56 \\
\midrule
jester-3& (1000, 0.15) & 0.227 & 4.27 & 0.221 & 97.82 & 0.218 & 83.18\\
\ & (1000, 0.20) & 0.220 & 4.41 & 0.212 & 103.28 & 0.212 & 84.02\\
\ & (1000, 0.25) & 0.221 & 4.54 & 0.213 & 105.48 & 0.212 & 84.90\\
\ & (1500, 0.15) & 0.225 & 5.47 & 0.218 & 272.30 & 0.215 & 237.38\\
\ & (1500, 0.20) & 0.220 & 5.54 & 0.212 & 280.34 & 0.212 & 240.19 \\
\ & (1500, 0.25) & 0.218 & 5.69 & 0.208 & 284.05 & 0.211 & 241.21 \\
\ & (2000, 0.15) & 0.226 & 6.46 & 0.216 & 585.71 & 0.218 & 521.87\\
\ & (2000, 0.20) & 0.222 & 6.59 & 0.217 & 606.53 & 0.212 & 525.93 \\
\ & (2000, 0.25) & 0.218 & 6.70 & 0.211 & 614.04 & 0.210 & 526.78 \\
  \bottomrule\bottomrule 
\end{tabular}
\label{tab:jester}
\end{table}

 \begin{table}[ht]
\centering
\caption{Averaged normalized mean absolute error  and running time in seconds for different methods using Movie-100K dataset. 
%For the nuclear norm approach, we set $\mu = 2\times10^{-4}\|Y_S\|_F$. For the max-norm approach, we set $\lambda=0.5\norm{Y_S}_F$. For the hybrid approach, and we set  $\lambda=0.8\norm{Y_S}_F,\mu=1\times10^{-4}\lambda$.
The parameters are chosen to be the same as those in 
Table \ref{tab:jester}.
}
\medskip
\begin{tabular}{r*{7}{r}}
\toprule
\toprule
 \ &\multicolumn{2}{c}{Nuclear}  &\multicolumn{2}{c}{Max} &\multicolumn{2}{c}{Hybrid}
\\
%&\multicolumn{2}{c}{($\lambda = 2\times10^{-4}\|Y_S\|_F$)} &\multicolumn{2}{c}{$(\lambda=0.1\norm{Y_S}_F,\mu=0)$} &\multicolumn{2}{c}{$(\lambda=0.2\norm{Y_S}_F,\mu=2\times10^{-4}\lambda)$}  \\
\cmidrule{2-3}  \cmidrule{4-5}  \cmidrule{6-7}
 \multicolumn{1}{r} {$\text{SR}$}& \multicolumn{1}{c}{NMAE} &  \multicolumn{1}{r}{Time} & \multicolumn{1}{c}  {NMAE} & \multicolumn{1}{r} {Time}  & \multicolumn{1}{c}  {NMAE} & \multicolumn{1}{r} {Time}\\
\midrule
0.10 & 0.243 & 108.4 & 0.231 & 266.8 & 0.232 & 292.2 \\
0.15 & 0.235 & 112.5 & 0.222 & 274.9 & 0.223 & 288.9 \\
0.20 & 0.233 & 112.1 & 0.213 & 263.4 & 0.220 & 286.2 \\
0.25 & 0.223 & 123.8 & 0.208 & 285.5 & 0.215 & 294.7 \\  
  \bottomrule\bottomrule 
\end{tabular}
\label{tab:movie}
\end{table}

 \begin{table}[ht]
\centering
\caption{Averaged normalized mean absolute error  and running time in seconds for different methods using Movie-1M dataset. 
%For the nuclear norm approach, we set $\mu = 2\times10^{-4}\|Y_S\|_F$. For the max-norm approach, we set $\lambda=0.5\norm{Y_S}_F$. For the hybrid approach, and we set  $\lambda=0.8\norm{Y_S}_F,\mu=1\times10^{-4}\lambda$.
The parameters are chosen to be the same as those in 
Table \ref{tab:jester}.
}
\medskip
\begin{tabular}{r*{8}{r}}
\toprule
\toprule
\ & \ &\multicolumn{2}{c}{Nuclear}  &\multicolumn{2}{c}{Max} &\multicolumn{2}{c}{Hybrid}
\\
%&\multicolumn{2}{c}{($\lambda = 2\times10^{-4}\|Y_S\|_F$)} &\multicolumn{2}{c}{$(\lambda=0.1\norm{Y_S}_F,\mu=0)$} &\multicolumn{2}{c}{$(\lambda=0.2\norm{Y_S}_F,\mu=2\times10^{-4}\lambda)$}  \\
\cmidrule{3-4}  \cmidrule{5-6}  \cmidrule{7-8}
\multicolumn{1}{r}{$n$} & \multicolumn{1}{r} {$\text{SR}$}& \multicolumn{1}{c}{NMAE} &  \multicolumn{1}{r}{Time} & \multicolumn{1}{c}  {NMAE} & \multicolumn{1}{r} {Time}  & \multicolumn{1}{c}  {NMAE} & \multicolumn{1}{r} {Time}\\
\midrule
1500 &0.10 & 0.248 & 154.7 & 0.235   & 377.6 &  0.236 & 409.2 \\
&0.15 & 0.238 & 154.1 & 0.222  & 318.3 & 0.229  & 410.9 \\
&0.20 & 0.233 & 153.9 & 0.216   & 329.8 & 0.223   &  401.9 \\
&0.25 & 0.225 & 210.7 & 0.208 &  473.3 & 0.218 & 506.2 \\  
\midrule
2000 &0.10 & 0.244 & 357.8 & 0.227   & 733.2 &  0.230 & 956.9 \\
&0.15 & 0.234 & 363.5 & 0.214  & 725.7 & 0.213 & 946.0 \\
&0.20 & 0.230 & 365.6 & 0.206   & 782.6 & 0.206   &  946.3 \\
&0.25 & 0.220 & 391.9 & 0.199 &  744.4 & 0.210 & 950.7 \\  
  \bottomrule\bottomrule 
\end{tabular}
\label{tab:movie2}
\end{table}

\section{Conclusions} \label{sec:con}
We propose a new matrix completion method using a hybrid nuclear- and max-norm regularizer. Compared with the standard nuclear-norm based approach, our method is adaptive under different sampling schemes and achieves fast rates of convergence. To handle the computational challenge, we propose the first scalable algorithm with provable convergence guarantee.  This bridges the gap between theory and practice of the max-norm approach.  In addition, we provide thorough numerical results to backup the developed theory.  This work paves the way for more potential machine learning applications of max-norm regularization.

A possible future direction  is to further improve the computational efficiency. 
The most computationally expensive component in Algorithm 1 is the $X$-update step, in which an eigenvalue decomposition is needed. By  solving some approximate version of this subproblem, it is possible to further boost the empirical performance and solve  problems of larger sizes.

\bibliographystyle{ims}
\bibliography{bib}

\renewcommand{\baselinestretch}{1.05}

\newpage
\begin{appendices}
\section{Extensions} \label{app:con}
In this section, we consider solving the max-norm constrained version of the optimization problem \eqref{eqn:hybrid}. In particular, we consider
\begin{equation} \label{eqn:maxc}
\min_{M \in \RR^{d_1\times d_2} }  \frac{1}{2}\sum_{t=1}^n \big( Y_{i_t,j_t} - M_{i_t,j_t}  \big)^2 + \langle M, I\rangle, \text{ subject to }\|M\|_\infty\le \alpha, \|M\|_{\max}\le R.
\end{equation}
This problem can be formulated as an SDP problem %that
{as follows:}
\begin{equation}\label{eqn:sdpc}
\begin{aligned}
\min_{Z\in \RR^{d \times d}} &~\frac{1}{2}\sum_{t=1}^n(Y_{i_t,j_t} - Z^{12}_{i_t,j_t})^2 + {\mu \inprod{I}{Z}},
\\
\text{subject to }&~ \|Z^{12}\|_\infty\le \alpha, \ \|\diag(Z)\|_\infty\le R, \ \ Z\succeq 0.
\end{aligned}
\end{equation}
Let the loss function  be 
$$
\cL(Z) = \frac{1}{2} \sum_{t=1}^n  \big( Y_{i_t,j_t} - Z^{12}_{i_t,j_t}  \big)^2 + \mu\langle I, Z\rangle.
$$
We define {the} set 
$$
\cP = \{Z\in\cS^d: \diag(Z)\ge 0, \|Z^{11}\|_{\infty} \le R, \|Z^{22}\|_\infty\le R, \|Z^{12}\|_\infty<\alpha\}.
$$
Thus, we have an equivalent formulation of \eqref{eqn:sdpc} below, which is more conducive for computation:
\begin{equation} \label{eqn:hybridc}
\min_{X,Z}\cL(Z) + \mu\langle X, I\rangle, \text{ subject to }X\succeq 0, \ Z\in\cP, X-Z=0.
\end{equation} 
We consider the augmented Lagrangian function of \eqref{eqn:hybridc} %that
{defined by}
$$
L(X,Z;W) =\cL(Z) +  \langle W,X-Z\rangle + \frac{\rho}{2} \|X-Z\|_F^2, \ X\in\cS^d_+, \ Z\in\cP,
$$
where $W$ is the dual variable. 
Then, it is natural to apply the ADMM to solve the problem \eqref{eqn:hybridc}. At the $t$-th iteration, we update $(X,Z;W)$ by
\begin{equation} \label{eqn:admmc}
\begin{aligned}
X^{t+1} &= \argmin_{X\in\cS^d_+} L(X,Z^t;W^t) 
= \Pi_{\cS_+^{d}}\big\{Z^t -\rho^{-1}{(W^t+\mu I)}\big\},\\
Z^{t+1}& = \argmin_{Z\in\cP} L(X^{t+1},Z;W^t) = 
\argmin_{Z\in\cP} \cL(Z) +\frac{\rho}{2} \|Z - X^{t+1}- \rho^{-1}W^t\|_F^2,\\
W^{t+1}& = W^{t} + {\tau} \rho(X^{t+1}-Z^{t+1}),
\end{aligned}
\end{equation}
The next proposition provides a closed-form solution for the $Z$-subproblem in \eqref{eqn:admmc}.
\begin{prp}
Denote the observed set of indices of $M^0$ by $\Omega = \{(i_t,j_t)\}_{t=1}^n$.
%Consider the optimization problem
%$
%\min_{Z}
%$$
 For a given matrix $C\in\RR^{d\times d}$, we have
\begin{equation}\label{eqn:ZC}
\cZ(C) = \argmin_{Z\in\cP} \cL(Z) + \frac{\rho}{2} \|Z-C\|_F^2,
\end{equation}
where
$$
\begin{aligned}
&\cZ(C) = \begin{pmatrix}
\cZ^{11}(C)&\cZ^{12}(C)\\
\cZ^{12}(C)^T&\cZ^{22}(C)
\end{pmatrix}\\
&\cZ_{k \ell}^{12}(C) =
\begin{cases}
\Pi_{[-\alpha,\alpha]} \Big(\frac{Y_{k \ell} + \rho C_{k \ell}^{12}}{\rho}\Big), &\quad\text{if }(k, \ell )\in S,\\
\Pi_{[-\alpha,\alpha]} (C_{k \ell}^{12}),&\quad \text{otherwise,}
\end{cases} \\
&\cZ_{k \ell}^{11}(C) = 
\begin{cases}
\Pi_{[-R,R]}\big(C_{k \ell}^{11}\big)\quad & \text{ if } k\neq \ell, \\
\Pi_{[0,R]}\big(C_{k\ell}^{11}\big)\quad &\text{ if } k = \ell,
\end{cases}
\quad
\cZ_{k \ell}^{22}(C) = 
\begin{cases}
\Pi_{[-R,R]}\big(C_{k \ell}^{22}\big)\quad & \text{ if } k\neq \ell, \\
\Pi_{[0,R]}\big(C_{k\ell}^{22}\big)\quad &\text{ if } k = \ell,
\end{cases}
% \diag\big\{\cZ(C)\big\} = \argmin_{z\in\RR^d}  \lambda\|z\|_\infty + \frac{\rho}{2}\|\diag(C) - z\|^2_2,
\end{aligned}
$$
and $\Pi_{[a,b]}(x) = \min\{b,\max(a,x) \}$ projects $x\in\RR$ to the interval $[a,b]$.
\end{prp}

We summarize the algorithm for solving the problem \eqref{eqn:sdpc} below.
 \begin{algorithm}[htb]
\caption{Solving max-norm optimization problem (\ref{eqn:sdpc}) by the ADMM}
\label{alg:maxnorm1}
Initialize $X^0$, $Z^0$, $W^0$, $\rho$, $\lambda$.
\begin{algorithmic}
\REQUIRE $X^0$, $Z^0$, $W^0$, $Y_\Omega$, $\lambda$, $R$, $\alpha$, $\rho$, $\tau$, $t=0$.
\WHILE {Stopping criterion is not satisfied.}
\STATE Update $X^{t+1}\leftarrow\Pi_{\cS_+^{d}}(Z^t -\rho^{-1}{(W^t+\mu I)})$.
\STATE Update $Z^{t+1}\leftarrow\cZ(X^{t+1} +\rho^{-1}W^t)$ by \eqref{eqn:ZC}.
\STATE Update $W^{t+1}\leftarrow W^{t} + {\tau}\rho(X^{t+1}-Z^{t+1})$.
\STATE $t\leftarrow t+1$.
\ENDWHILE
\ENSURE $\hat{Z} = Z^{t}$, $\hat{M} = \hat{Z}^{12}\in\RR^{d_1\times d_2}$.
\end{algorithmic}
\end{algorithm}
\end{appendices}

\end{document}